\documentclass{article}

\usepackage{microtype}
\usepackage{graphicx}
\usepackage{subfigure}
\usepackage{booktabs}
\usepackage{subcaption}

\usepackage{hyperref}

\newcommand{\Reals}{\mathbb{Re}}

\newcommand{\indep}{\perp\!\!\!\perp}
\newcommand{\notindep}{\not\!\perp\!\!\!\perp}
\newcommand{\cov}{\mathrm{Cov}}
\newcommand{\Pa}{\mathrm{Pa}}
\newcommand{\Ch}{\mathrm{Ch}}
\newcommand{\Dec}{\mathrm{De}}

\newcommand{\MB}{\mathrm{MB}}

\newcommand{\sepset}{\mathrm{SepSet}}

\usepackage[accepted]{icml2024}

\usepackage{amsmath}
\usepackage{amssymb}
\usepackage{mathtools}
\usepackage{amsthm}
\usepackage{bbm}

\usepackage[capitalize,noabbrev]{cleveref}

\theoremstyle{plain}
\newtheorem{theorem}{Theorem}[section]
\newtheorem{proposition}[theorem]{Proposition}

\theoremstyle{definition}
\newtheorem{definition}[theorem]{Definition}
\newtheorem{assumption}[theorem]{Assumption}
\newtheorem{example}[theorem]{Example}
\newtheorem{problem}[theorem]{Problem}
\theoremstyle{remark}

\usepackage[textsize=tiny]{todonotes}

\icmltitlerunning{Scalable and Flexible Causal Discovery}

\begin{document}

\twocolumn[
\icmltitle{Scalable and Flexible Causal Discovery with an Efficient Test for Adjacency}

\begin{icmlauthorlist}
\icmlauthor{Alan Nawzad Amin}{yyy}
\icmlauthor{Andrew Gordon Wilson}{yyy}
\end{icmlauthorlist}

\icmlaffiliation{yyy}{New York University, New York, USA}

\icmlcorrespondingauthor{Alan Nawzad Amin}{alanamin@nyu.edu}

\icmlkeywords{Machine Learning, ICML}

\vskip 0.3in
]

\printAffiliationsAndNotice{} 
\begin{abstract}
To make accurate predictions, understand mechanisms, and design interventions in systems of many variables, we wish to learn causal graphs from large scale data.
Unfortunately the space of all possible causal graphs is enormous so scalably and accurately searching for the best fit to the data is a challenge.
In principle we could substantially decrease the search space, or learn the graph entirely, by testing the conditional independence of variables.
However, deciding if two variables are adjacent in a causal graph may require an exponential number of tests.
Here we build a scalable and flexible method to evaluate if two variables are adjacent in a causal graph, the Differentiable Adjacency Test (DAT).
DAT replaces an exponential number of tests with a provably equivalent relaxed problem.
It then solves this problem by training two neural networks.
We build a graph learning method based on DAT, DAT-Graph, that can also learn from data with interventions.
DAT-Graph can learn graphs of 1000 variables with state of the art accuracy.
Using the graph learned by DAT-Graph, we also build models that make much more accurate predictions of the effects of interventions on large scale RNA sequencing data.
\end{abstract}

\section{Introduction}
\label{intro}
Large scale studies have recently collected hundreds of thousands of measurements of thousands of variables and interventions across genetics, microbiology, and healthcare \cite{Van_Hout2020-rg, Regev2017-jk, Franzosa2019-us, Geiger-Schuller2023-ig, Replogle2022-cy, Dixit2016-yl}.
An algorithm that leverages this data to learn cause and effect must scale to many measurements and variables, flexibly accommodate complex interactions between variables, and make reliable predictions in realistic and large scale settings.

Modern state-of-the-art algorithms frame learning causal relationships as a model selection problem \cite{Chickering2002-rx, Van_de_Geer2013-ak}.
Each model corresponds to a directed acyclic graph representing which variables cause which others.
Complex relationships between variables are then modelled using flexible neural networks \cite{Lachapelle2019-jl, Zheng2020-wl}.
The central practical challenge of this approach is the model search, where one needs to explicitly search through the enormous space of all directed acyclic graphs. 
Recently, a number of gradient-based search methods have scaled this procedure to data of large complex systems \cite{Zheng2018-dc, Lachapelle2019-jl, Zheng2020-wl, Nazaret2023-yy}.
However, these heuristic search procedures can be unstable in practice and can be unreliable even in simple settings \cite{Wei2020-ju, Nazaret2023-yy, Deng2023-ix}. The model search also becomes exponentially harder as the number of variables increases.

To make more accurate predictions at scale we can shrink the search space, or avoid searching by learning the graph entirely, by taking the alternative approach of many classical graph learning algorithms. These approaches test the data for conditional independence relationships, and then exclude the corresponding edges from the graph \cite{Spirtes1993-ty}.
Indeed recently, by reducing the model search space with some limited testing, a gradient-based model search strategy demonstrated large gains in accuracy on data of large complex systems \cite{Nazaret2023-yy}.
And historically, by reducing the search space by testing as much as possible before performing model search, classical non-flexible ``hybrid'' causal models achieved state-of-the-art accuracy learning from data of simple systems \cite{Tsamardinos2006-zk, Buhlmann2014-do}.

Unfortunately, classical testing-based procedures struggle to scalably and reliably learn from data of large complex systems.
A particularly informative test, and the central step in virtually all classical testing-based procedures, is evaluating if two variables are immediate causes or effects of each other --- that is, if they are adjacent in the causal graph \cite{Spirtes1993-ty, Tsamardinos2006-zk}.
Evaluating this relationship for a pair of variables involves searching for a set of other variables that renders them conditionally independent (Fig.~\ref{fig: conceptual a}). 
\begin{figure}
    \centering
    \subfigure[][\label{fig: conceptual a}]{\includegraphics[width=0.41\columnwidth]{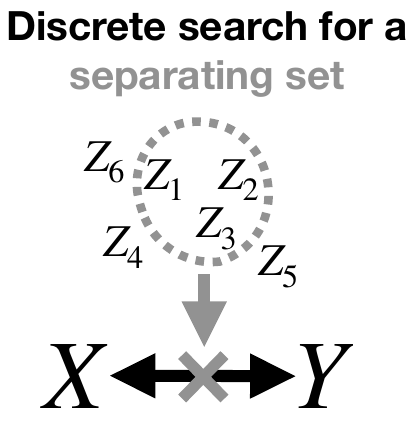}}
    \hspace{0.35cm}
    \subfigure[][\label{fig: conceptual b}]{\includegraphics[width=0.505\columnwidth]{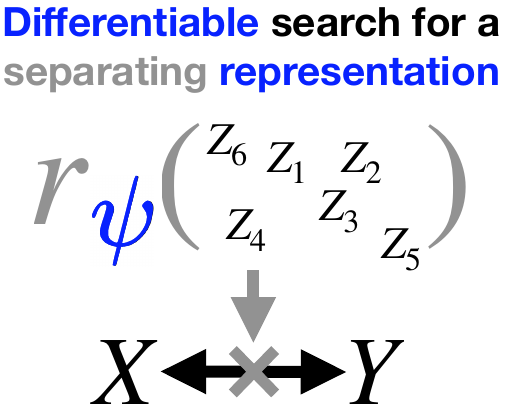}}
    \vspace{-.3cm}
    \caption{\textbf{Our Differentiable Adjacency Test (DAT) relaxes a discrete search requiring exponentially many tests (a) into a differentiable search to solve an optimization problem (b).}} \label{fig:conceptual}
    \vspace{-.6cm}
\end{figure}
For each pair this search can involve an exponential number of tests.
Because of the enormous number of tests, and the particularly large computational cost of flexible conditional independence tests \cite{Zhang2011-yw, Sen2017-op, Berrett2019-tx, Bellot2019-ln}, this procedure does not scale to data of large complex systems.
Furthermore, because realistic data can include many spurious borderline conditional independence relations --- so-called ``violations of faithfulness'' --- testing for adjacency can be unreliable even in simple settings \cite{Uhler2012-pz, Andersen2013-ak}.

Here we develop a testing-based graph learning procedure, DAT-Graph, that can scalably, flexibly, and reliably learn cause and effect from data of large and complex systems.
DAT-Graph is centred around a method for evaluating whether two variables are adjacent in the causal graph, DAT.
To do this flexibly and scalably, DAT replaces performing an exponential number of tests (Fig.~\ref{fig: conceptual a}) with optimizing a single differentiable objective (Fig.~\ref{fig: conceptual b}).
We build DAT carefully so that the testing and optimization problems are provably equivalent.
DAT-Graph learns the graph in a way that is reliable even when there are some violations of faithfulness;
in particular, since DAT is scalable, 
DAT-Graph can evaluate the adjacency of every pair of variables with two complementary tests in a way that is too computationally expensive for previous methods.
Empirically, we show that DAT-Graph is able to easily scale to learn on real and synthetic data with $10^3$ variables and $10^4$ observations.
We show DAT-Graph learns large sparse causal graphs more accurately than state of the art gradient-based model selection procedures with and without interventions.
We also use DAT-Graph to reduce the search space of model selection procedures to build even more accurate hybrid models.
We show that these hybrid models accurately predict the effects of interventions on large scale RNA sequencing data.

Our code is available at \url{https://github.com/AlanNawzadAmin/DAT-graph/}.

\vspace{-.2cm}
\section{Related work} \label{sec: related work}
\vspace{-.1cm}

There is an enormous history of methods to learn graphs in general.
We give a full review of these methods in Section~\ref{app: related work}.
Efforts to build testing-based graph learning procedures for data of large and complex systems have focused on building more flexible conditional independence tests using kernels \cite{Fukumizu2007-hv, Zhang2011-yw, Pogodin2022-dc}, simulation of conditional distributions \citep{Doran2014-gr, Sen2017-op, Berrett2019-tx, Bellot2019-ln}, or a different strategy \citep{Shah2020-vv, Polo2023-os, Laumann2023-nv};
or building more scalable tests \citep{Strobl2019-iv, Runge2018-tp} or reducing the number of tests by performing them in a clever order \citep{Margaritis1999-vu, Mokhtarian2020-lh}.
While these methods are accurate at small scale \citep{Pogodin2022-dc}, they unfortunately cannot scale to more than 100 variables.
DAT instead replaces an exponential number of tests with an equivalent relaxed problem.

The relaxed problem DAT considers is equivalent to the problem considered in Invariant Risk Minimization (IRM) \cite{Arjovsky2019-ml} when one of the variables is a discrete environment variable.
DAT however can accommodate continuous variables.
DAT also diverges strongly from IRM methodologically so that the answer to the relaxed problem is provably identical to the initial problem.

\vspace{-.3cm}
\section{Background}\label{sec: background}
\vspace{-.1cm}
We wish to learn causal relationships between $N$ real random variables $X^1, \dots, X^N$.
We model cause and effect relationships as functional relationships; that is, if we call $\Pa(X^n)$ the direct causes of the variable $X^n$ then we assume there is some function $h_n$ such that
\vspace{-0.5cm}

\begin{equation}\label{eq: functional}
    X^n = h_n\left(\epsilon_n, (X^m)_{m\in\Pa(X^n)}\right)
\end{equation}

\vspace{-0.2cm}
where $(\epsilon_n)_{n=1}^N$ are iid random variables independent from each other and $X^{1:N}$.
We represent causal relations in a graph $G$ with nodes $X^{1:N}$ where there is a directed edge from $X^n$ to $X^m$ if $X^m\in\Pa(X^n)$, that is, each variable $X^n$ is caused by its parents in $G$, $\Pa_G(X^n)$.
We assume that there are no cycles in this graph.

Now we wish to recover the graph $G$ from some observations $X_1^{1:N}, \dots, X_K^{1:N}$.
We assume the observations come iid from some distribution $p$, and defer observations with intervened variables until Section \ref{sec: interventions};
we call this the ``purely observational'' setting.
First note that in this setting the graph $G$ is only identifiable up to an equivalence class~\cite{Pearl2010-si}.
However, the skeleton of the graph -- a graph in which two nodes are connected by an undirected edge if they are adjacent in $G$ -- is identifiable \cite{Pearl2010-si}, and in many cases we can distinguish cause and effect for almost all adjacent variables in $G$ \cite{Katz2019-sw}.
Our goal is to learn a member of the equivalence class of $G$.

The causal relationships in Eqn~\ref{eq: functional} necessarily imply observable conditional independence relationships.
If the variables $X^{1:N}$ follow Eqn~\ref{eq: functional} then it must be the case that $X^n$, when conditioned on its parents $\{X^m\}_{m\in \Pa_G(X^n)}$, is independent of all nodes other than its decendents in $G$, $\Dec_G(X^n)$:
\vspace{-0.9cm}

\begin{equation}\label{eq: markov eqn}
    X^n\indep \{X^m\}_{m\not\in \Dec_G(X^n)}\ |\ \{X^m\}_{m\in \Pa_G(X^n)}.
\end{equation}

\vspace{-0.2cm}
These relationships in turn necessarily imply that any two sets of nodes $A, B\subset X^{1:N}$ that are d-separated in $G$ by a third set of nodes $C\subset X^{1:N}$ are conditionally independent $A\indep B\ |\ C$ \cite{Geiger1990-ix} (see \citet{Spirtes1993-ty} for a review of d-separation).
In the generic case, these are all of the conditional independence relationships of the nodes $X^{1:N}$, meaning that $p$ is faithful to the graph $G$ \cite{Geiger1990-ix, Uhler2012-pz}.
\begin{definition}
    A distribution $p$ over $X^{1:N}$ is faithful to a graph $G$ if for any three disjoint sets $A, B, C\subset X^{1:N}$, $A\indep B\ |\ C$ if and only if $C$ d-separates $A$ and $B$ in $G$.
\end{definition}
\vspace{-.2cm}
Our goal is to look for these conditional independence relationships and thereby learn about the topology of the causal graph.
Note however that even when $p$ is faithful to a graph $G$, there are in practice many subsets $A, B, C$ that nearly violate faithfulness -- that is, $C$ may not d-separate $A$ and $B$ but $A$ is ``almost'' independent of $B$ when conditioned on $C$ \cite{Uhler2012-pz, Andersen2013-ak}.
Therefore, to build a reliable method we would like to be robust to the presence of a few near violations of faithfulness.

A crucial observation is that two non-adjacent nodes in a directed acyclic graph can always be d-separated by some other set of nodes, while adjacent nodes can never be d-separated.
Thus if $p$ is faithful to $G$ then two nodes $X^n, X^m$ are adjacent if and only if there is no other set $S\subset \{1, \dots, N\}\setminus\{n, m\}$ such that $X^n\indep X^m\ |\ \{X^k\}_{k\in S}.$
Thus we can learn the skeleton of $G$ if we had a method to solve what we call the separating set selection problem.
\begin{problem}\label{prob: sepset}
    \textbf{(Separating set selection Fig.~\ref{fig:conceptual}(a))}
    Given a set of real random variables $X, Y, Z_1, \dots, Z_M$, is there a subset $S\subset\{1, \dots, M\}$ such that $X\indep Y|\{Z_m\}_{m\in S}$?
\end{problem}
\vspace{-.2cm}
If $\{Z_m\}_{m\in S}$ is any subset such that $X\indep Y|\{Z_m\}_{m\in S}$ then we call it a separating set of $X, Y$: $\sepset(X, Y)=\{Z_m\}_{m\in S}$.
If $p$ is faithful to $G$, then if we have the skeleton of $G$ and $\sepset(X^n, X^m)$ for any pair of nonadjacent nodes $X^n, X^m$ then we can determine the equivalence class of $G$ according to a set of rules \cite{Spirtes1993-ty}.
Thus, to learn $G$ from data of large complex systems all we need is a scalable method to solve the separating set selection problem that can flexibly represent complex relationships between $X, Y, Z_1, \dots, Z_m$.
\vspace{-.3cm}

\section{The differentiable adjacency test (DAT)} \label{sec: dat}
\vspace{-.1cm}
In this section we build a scalable, flexible, and reliable method to solve the separating set selection problem.
In Section~\ref{sec: discrete to diff search} we first relax the separating set selection problem to a differentiable problem -- the separating representation search problem.
We prove conditions under which the relaxed problem answers the separating set selection problem.
In Section~\ref{sec:NP-hardness} we prove that the separating set selection problem is NP-Hard so we unfortunately cannot guarantee that there are not cases where it is challenging to answer the relaxed problem.
Nevertheless, in Section~\ref{sec: diff objective} we build the Differentiable Adjacency Test (DAT), a practical method to efficiently approximately answer the separating representation search problem using neural networks to represent complex relations between variables.
Finally in Section~\ref{sec: dat results} we show the DAT solves the separating set problem as accurately as a classical testing approach while being orders of magnitude faster.
Throughout we use the term ``testing'' in the informal sense of a decision rule, without implying validity or coverage.

\vspace{-.3cm}
\subsection{Relaxing the discrete search}\label{sec: discrete to diff search}
\vspace{-.1cm}
The computational challenge of the separating set selection problem is the discrete search over all $2^M$ subsets of $\{1, \dots, M\}$.
We can relax the problem by replacing a search over subsets of $Z_{1:M}$ with a search for a representation $r_{\psi^*}(Z_{1:M})$
where $\{r_\psi\}_{\psi\in\Psi}$ is a differentiably parameterized family of functions that have domain $\Reals^M$.
\begin{problem}\label{prob: sepset diff}
    \textbf{(Separating representation search Fig.~\ref{fig:conceptual}(b))}
    Given a set of real random variables $X, Y, Z_1, \dots, Z_M$ and a class of possibly random functions $\{r_\psi\}_{\psi\in\Psi}$, is there a $\psi^*\in \Psi$ such that $X\indep Y|r_{\psi^*}(Z_{1:M})$?
\end{problem}
\vspace{-.2cm}
Indeed, the separating representation search problem is a natural relaxation of the separating set selection problem that has come up in other causal inference settings in the case that $Y$ is a discrete ``environment'' variable \cite{Arjovsky2019-ml, Shi2020-ji}.

In our setting we need to pick our representations $\{r_\psi\}_{\psi\in\Psi}$ so that 1) there is a separating representation if and only if there is a separating set and 2) we can get a separating set $\sepset(X, Y)$ from the separating representation parameter $\psi^*$.
Unfortunately, there are seemingly reasonable choices in simple situations in which these desiderata are not fulfilled.
\begin{example}\label{ex: r problem}
    \textbf{(Existence of a separating representation but no separating set)}
    There are jointly Gaussian variables $X, Y, Z_1, Z_2$ that are faithful to some graph such that \mbox{$X\notindep Y|\{Z_m\}_{m\in S}$} for any $S\subset\{1, 2\}$ but if $\{r_\psi\}_{\psi\in\Psi}$ is the space of linear functions, there is a $\psi^*$ such that $X\indep Y|r_{\psi^*}(Z_{1:2})$.
    \vspace{-.4cm}
\end{example}
\begin{proof}
    In Appendix~\ref{app: counterexamples}.
    \vspace{-.3cm}
\end{proof}
Unfortunately, by relaxing the problem, we have in effect increased the number of opportunities for there to be a violation of faithfulness from an exponential number -- checking $X\indep Y|\{Z_m\}_{m\in S}$ for all subsets $S$ -- to an infinite number -- checking \mbox{$X\indep Y|r_{\psi}(Z_m)_{m=1}^M$} for all $\psi$.

To build a method with our desiderata, our strategy will be to restrict what $\{r_{\psi}\}_{\psi}$ can represent.
We choose $r_\psi$ to only represent ``soft'' subsets of $Z_{1:M}$ where each variable $Z_m$ is softly included in the separating set by mixing it with independent noise.
Let $N_1, \dots, N_M$ be drawn independently from distributions with densities $f_1, \dots, f_M$.
For $\psi=(\psi_1, \dots, \psi_M)$ with $\psi_m\in[0, 1]$, we define the noised variable
$$\tilde Z_{\psi, m} = \psi_m Z_m + (1-\psi_m)N_m$$
and the representation $r_\psi(Z_{1:M})=\tilde Z_{\psi, 1:M}$.
By mixing $Z_m$ with an independent random variable $N_m$ we lose information about $Z_m$.
$\psi_m$ controls how much information we observe about $Z_m$; when $\psi_m=0$ we do not observe $Z_m$ and when $\psi_m=1$ we observe $Z_m$ fully.

With this choice, one can still design $f_1, \dots, f_M$ to get disagreeing answers between the separating set selection and separating representation search problems (see Example~\ref{ex: noise problem}).
However we prove that if we pick $f_1, \dots, f_M$ to have thick tails then the answer to the two problems will be identical and we can recover a separating set from a separating representation as desired.
\begin{theorem}\label{Thm: main reliability}
    (Proof in Appendix~\ref{app: main proof})
    Assume Assumption~\ref{ass: tail assump} (pick $f_m$ to have thicker tails than $p$).
    The separating set selection problem and the separating representation search problem have the same answer.
    If $\tilde Z_{\psi^*, 1:M}$ is a separating representation then $\{Z_m\}_{\psi_m^*=1}$ is a separating set.
\end{theorem}
\vspace{-.3cm}

\subsection{Hardness of adjacency testing}\label{sec:NP-hardness}
\vspace{-.1cm}
In Thm.~\ref{Thm: main reliability} we showed that if we find a separating representation with parameter $\psi^*$ then we obtain a separating set.
Can we build an efficient method in practice that is guaranteed to find $\psi^*$?
We answer this question negatively.
\begin{proposition}\label{prop: np hard}
    (Proof in Appendix~\ref{app: proof np hard})
    Even when restricted to the case where $X, Y, Z_1, \dots,Z_M$ are jointly Gaussian with known non-singular covariance matrix, the separating set selection problem is NP-Hard.
    \vspace{-.2cm}
\end{proposition}

There may therefore be cases where, by failing to find $\psi^*$, we may incorrectly determine that two variables are adjacent in $G$.
Nevertheless, we aim to provide a method that gives accurate approximate solutions with reasonable compute.
\vspace{-.2cm}

\subsection{Differentiable Adjacency Test}\label{sec: diff objective}
\vspace{-.1cm}
We now build a method to search for a separating representation in practice.
Our method, DAT, will perform this search by minimizing a differentiable objective.
\vspace{-.1cm}

To replace the separating representation search problem with a differentiable optimization problem, we need to replace $X\indep Y|\tilde Z_{\psi}$, where we write $\tilde Z_{\psi}=\tilde Z_{\psi, 1:M}$, with a differentiable objective $L(\psi)$ that reaches its minimum if and only if $X\indep Y|\tilde Z_{\psi}$.
While a number of flexible measures of conditional independence exist \cite{Zhang2011-yw, Pogodin2022-dc, Bellot2019-ln}, we pick $L(\psi)$ to be the ``variance of $X$ explained by $Y$ when conditioned on $\tilde Z_{\psi}$'' as it is a well studied measure of conditional independence \cite{Zhang2020-kl, Polo2023-os} that is easy to optimize in practice and allows us to model the relationships between variables with scalable and flexible neural networks:

\vspace{-0.8cm}
\begin{equation*}
    \begin{aligned}
        E\mathbb{V}(X; Y|&\tilde Z_{\psi})=\\
        &E[X - E[X|\tilde Z_{\psi}]]^2 - E[X-E[X|Y, \tilde Z_{\psi}]]^2.
    \end{aligned}
\end{equation*}

\vspace{-.4cm}
This quantity is always non-negative and if $X\indep Y|Z$ then it is $0$.
On the other hand, if $E\mathbb{V}(T(X); Y|\tilde Z_{\psi})=0$ for all bounded functions $T$, then $X\indep Y|Z$.
We increase the ability of this metric to detect that \mbox{$X\notindep Y|\tilde Z_{\psi}$} by evaluating the variance explained of more than one statistic of $X$, $T_1(X), \dots, T_C(X)$,
\vspace{-0.5cm}

\begin{equation}
    \begin{aligned}
        \sum_{c=1}^CE\mathbb{V}(T_c(X); Y|\tilde Z_{\psi}).
    \end{aligned}
\end{equation}

\vspace{-0.3cm}
In experiments we choose $C=2$ and $T_1, T_2$ as the first and second moments of $X$.
For clarity we write the rest of the section as if $C=1$ and $T_1(X)=X$.
\vspace{-.1cm}

The variance explained cannot be exactly calculated, so we must approximate it.
We first approximate $E[X|\tilde Z_{\psi}]$ with a neural network $g_{\theta_1}:\Reals^M\to \Reals$ by training it to minimize the objective
\vspace{-0.7cm}

\begin{equation*}\label{eq:aprox objective g1}
    \begin{aligned}
        L_1(\theta_1) = E\left[X-g_{\theta_1}(\tilde Z_{\psi})\right]^2.
    \end{aligned}
\end{equation*}

\vspace{-0.3cm}
Then, calling the residue $R=X-E[X|\tilde Z_{\psi}]\approx X - g_{\theta_1}(\tilde Z_{\psi})$, we approximate $E[R|Y, Z]$ with a neural network $g_{\theta_2}:\Reals^{M+1}\to\Reals$ by training it to minimize the objective
\vspace{-0.4cm}

\begin{equation*}\label{eq:aprox objective g2}
    \begin{aligned}
        L_2(\theta_2) = E\left[X - g_{\theta_1}(\tilde Z_{\psi}) -g_{\theta_2}(Y, \tilde Z_{\psi})\right]^2.
    \end{aligned}
\end{equation*}

\vspace{-0.2cm}
Finally we optimize $\psi$ to minimize the approximate variance explained,
\vspace{-0.7cm}

\begin{equation*}\label{eq:aprox objective r}
    \begin{aligned}
        &E\mathbb{V}(X; Y|\tilde Z_{\psi})=\\
        =&E[X - E[X|\tilde Z_{\psi}]]^2 - E[R-E[R|Y, \tilde Z_{\psi}]]^2\\
        \approx &E[X-g_{\theta_1}(\tilde Z_{\psi})]^2 - E[X-g_{\theta_1}(\tilde Z_{\psi})-g_{\theta_2}(Y, \tilde Z_{\psi})]^2\\
        =&L_{\mathrm{DAT}}(\psi).
    \end{aligned}
\end{equation*}

\vspace{-0.4cm}
We call this optimization problem the Differentiable Adjacency Test (DAT).
We optimize by gradient descent, alternately updating $\theta_1, \theta_2, \psi$ by taking gradient steps $\nabla_{\theta_1}L_1, \nabla_{\theta_2}L_2, \nabla_{\psi}L_{\mathrm{DAT}}$ and approximating the expectations with mini-batches of the data.
For every data point in a mini-batch we draw independent noise $N_1, \dots, N_M$ to calculate $\tilde Z_\psi$.
Training can be framed as a two player game between $\{\theta_1, \theta_2\}$ and $\psi$, which is not guaranteed to converge to an optimum.
In practice however, this is not a challenging optimization problem as the number of variables of one of the players $\psi_{1:M}$ is small
    -- we do not observe the challenges with optimizing multiplayer games such as cycles or instability.
\vspace{-0.1cm}

Once we have finished training $\theta_1^*, \theta_2^*, \psi^*$, in theory we need to evaluate if $L_{\mathrm{DAT}}(\psi^*)=0$ and choose $\sepset(X, Y)=\{Z_m\}_{\psi_m=1}$ if so.
In practice, we first choose two thresholds $\eta_1, \eta_2>0$, and we use a mini-batch of data to get an approximation of the variance explained $\hat L_{\mathrm{DAT}}(\psi^*)$.
Then we decide $X\indep Y|\tilde Z_{\psi^*}$ if and only if $\hat L_{\mathrm{DAT}}(\psi^*)<\eta_1$ and we set $\sepset(X, Y)=\{Z_m\}_{\psi_m>\eta_2}$
\vspace{-0.2cm}
\subsection{DAT is accurate and efficient in practice}\label{sec: dat results}
\vspace{-0.1cm}
Before using the DAT to learn an entire graph, we evaluate its ability to solve the separating set selection problem.
We compare the DAT with the classical method to solve the separating set selection problem: test whether $X$ is independent of $Y$ given $S$ for each subset $S$ of $Z_{1:M}$ and conclude that there is a separating set if the maximum of a test statistic across all tests is above a threshold; then infer that the $S$ for which the test statistic is maximized is a separating set.

We generated small Erdos-Renyi random graphs such that each node had an average of one parent and selected $X$ and $Y$ to be two random nodes and $Z_{1:M}$ to be all other nodes.
We then generated 10000 data points from the graph with two layer neural networks as functional relationships as described in Section~\ref{sec: experiments}.
We consider five conditional independence tests: 
\textbf{RCoT} \citep{Strobl2019-iv}:
a scalable and flexible kernel test for conditional independence.
\textbf{GCIT} \citep{Bellot2019-ln}: A flexible simulation-based conditional independence testing method; it learns a conditional distribution by training a GAN. 
\textbf{CCIT} \citep{Sen2017-op}: A scalable and flexible simulation based independence testing method; it simulates a conditional distribution by a nearest-neighbors search. 
\textbf{CMIT} \citep{Runge2018-tp}: A flexible method based on estimating conditional mutual information by looking for nearest neighbors. 
\textbf{AT\_discrete}: A method to test the conditional independence of two variables by estimating the conditional variance explained; it is exactly DAT without a differentiable search. 
Due to compute limitations, we were only able to run GCIT, CMIT, and AT\_discrete to $M=3$, while CCIT could scale to $M=7$ and RCoT to $M=11$.

\begin{figure}
    \centering
    \includegraphics[trim={0.6cm 0.cm 0.6cm 0.cm}, width=0.9\columnwidth]{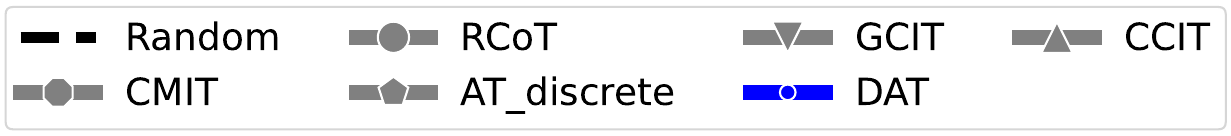}
    \subfigure[][\label{fig: dat acc} DAT accuracy]{
    \includegraphics[trim={0.6cm 0.6cm 0.cm 0.6cm}, width=0.44\columnwidth]{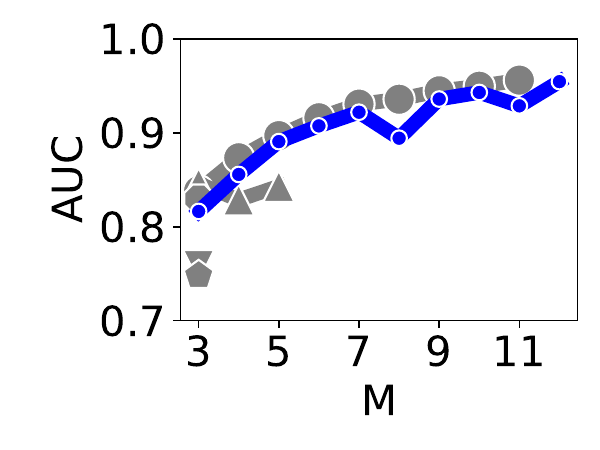}
    \includegraphics[trim={0.6cm 0.6cm 0.cm 0.6cm}, width=0.44\columnwidth]{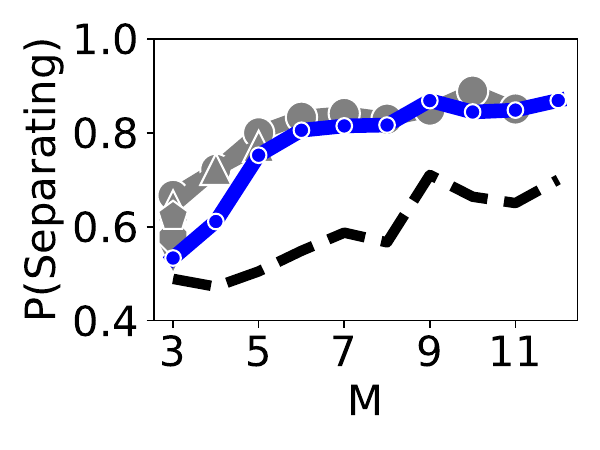}
    }
    \subfigure[][\label{fig: dat sep set} DAT efficiency]{
    \hspace{0.1cm}
    \includegraphics[trim={0.16cm 0.55cm 0.1cm 0.3cm}, width=0.44\columnwidth]{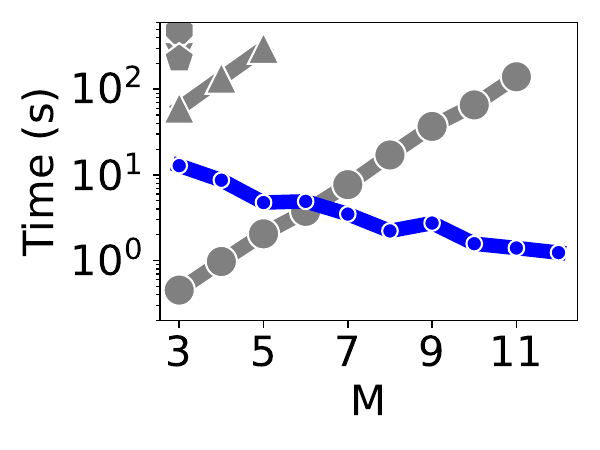}
    \hspace{-0.05cm}
    \includegraphics[trim={0.6cm 0.45cm 0.1cm 0.3cm}, width=0.434\columnwidth]
    {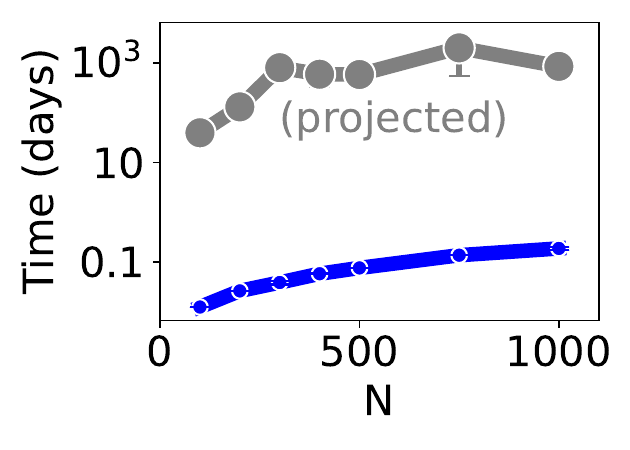}
    \hspace{0.05cm}
    }
    \vspace{-.3cm}
    \caption{\textbf{DAT enables learning large graphs by solving the separating set selection problem accurately and efficiently.} 
    (a) We plot how accurately each method determines if two variables are adjacent (AUC) and how often it corretly identifies a separating set for two non-adjacent variables (P(Separating)) against number of variables ($M$).
    (b) We plot the time of running each method against number of variables ($M$) and the time it would take to use each method to learn a large graph against the size of a graph ($N$).
    We plot the mean and standard error across 3 replicates. 
    } \label{fig: dat acc and scale}
    \vspace{-.6cm}
\end{figure}

We benchmark the accuracy of each method in classifying adjacent variables by the Area Under the receiver operator Curve (AUC) of the test statistic.
In Fig.~\ref{fig: dat acc} (left), we show that DAT is nearly as accurate as state of the art conditional independence tests in determining if two variables are adjacent.
In Fig.~\ref{fig: dat acc} (right) we see that DAT also identifies a separating set nearly as accurately as classical methods.

In Fig.~\ref{fig: dat sep set} (left) we show that all testing methods scale exponentially in compute with $M$ while DAT does not\footnote{In theory DAT scales linearly with $M$ with a large constant overhead for data transfers to the GPU; for small $M$ however, DAT can parallelized testing many edges at once on a GPU and reduce the computation per edge.}.
Finally we show that the scalability of DAT enables us to accurately learn large graphs.
We extrapolate the exponential scaling of Fig.~\ref{fig: dat sep set} (right) of the most scalable test, RCoT, and project how much time it would take to learn large graphs swapping DAT with RCoT in our experiments in Section~\ref{sec: experiments} (here $M$ could reach 30). 
We see in Fig.~\ref{fig: dat sep set} (right) that it would take weeks or years to learn graphs with many variables $N$ with classical tests while DAT took minutes to hours in our experiments.
\vspace{-.3cm}

\section{Learning a graph with DAT (DAT-Graph)}\label{sec: dat-graph}
\vspace{-.1cm}
In this section we use DAT to build a scalable, flexible, and reliable method to learn a causal graph $G$ from data of large, complex systems.
We call this method DAT-Graph.

In principle we could learn the graph by testing the adjacency of every pair of nodes using DAT.
This would involve $N^2$ tests, each of which involves a search over $N-2$ variables.
This strategy is clearly not scalable, and is also unreliable under possible near violations of faithfulness.
Instead we use a strategy employed by a large number of hybrid and testing-based graph-learning methods to first efficiently and reliably exclude a large number of edges from $G$ \cite{Margaritis1999-vu, Mokhtarian2020-lh, Buhlmann2014-do, Nazaret2023-yy}.
The idea is to try and predict the distribution of each variable $X^n$ using all other variables -- if we then determine that a variable $X^m$ is not useful in predicting $X^n$ we can conclude that it cannot be connected to $X^n$ in $G$.
Then we only need to solve the separating set selection problem to test and orient the substantially reduced number of remaining edges.
This first step also reduces the search space of each separating set selection problem \cite{Margaritis1999-vu, Mokhtarian2020-lh}.

However following previous methods by performing just these tests still leaves DAT-Graph unreliable when there are near violations of faithfulness.
To greatly increase its reliability in theory and in practice, unlike previous methods, DAT-Graph evaluates the adjacency of each pair of nodes by solving two separating set selection problems with different search spaces; it is able to perform both tests efficiently due to the scalability of DAT.

In Section~\ref{sec: moral graph} we describe how we learn the moral graph -- the graph that represents which variables are useful for predicting which others.
In Section~\ref{sec: skel graph} we describe how learning the moral graph reduces the search space for each adjacency test; we also describe how we perform DAT twice to test the adjacency of each pair of variables.
In Section~\ref{sec: datgraph scaling} we discuss the computational complexity of DAT-Graph.

Once we have learned the skeleton of $G$, we need to orient its edges.
We review how to do so using the separating sets we have calculated from DAT and standard rules from classical testing methods \cite{Spirtes1993-ty} in Appendix~\ref{sec: full graph}.
\vspace{-.2cm}

\subsection{Learning the moral graph}\label{sec: moral graph}
\vspace{-.1cm}
The first step in DAT-Graph is to exclude edges between variables that are not useful in predicting each other.
To do so, we must identify, for each variable $X^n$, the sparsest set in $\{X^m\}_{m\neq n}$ that can predict the distribution of $X^n$.
This sparsest set is known as the Markov blanket of $X^n$.
\begin{definition}\label{def: moral graph}
    A Markov blanket of a variable $X^n$ is a smallest subset $\MB(X^n)\subset\{X^m\}_{m\neq n}$ that makes $X^n$ conditionally independent of all other variables
    $$X\indep \{X^m\}_{m\neq n}\setminus \MB(X^n)\ |\ \MB(X^n).$$
    \vspace{-.8cm}
\end{definition}
The moral graph is the graph with an undirected edge between variables $X^n$ and $X^m$ if $X^n\in\MB(X^m)$.
When $p$ is faithful, the Markov blanket of a variable $X^n$ is a unique set consisting of all the variables that are adjacent to $X^n$ in $G$ as well as ``spouses'' of $X^n$ -- variables that are not adjacent to $X^n$ in $G$ but share a direct child with $X^n$ \cite{Spirtes1993-ty}.
Thus edges that are not in the moral graph also are also not in the skeleton of $G$.
If we can learn the moral graph we can therefore exclude many edges from $G$.

\citet{Schmidt2007-cr}, \citet{Buhlmann2014-do}, and \citet{Nazaret2023-yy} have shown that the Markov blanket of $X^n$ is the solution to any sparse variable selection procedure.
To solve sparse variable selection there are an enormous number of scalable, flexible, and reliable algorithms.
We adapt a variable selection method from \citet{Nazaret2023-yy} that allows us to flexibly model causal relationships using neural networks.
For each $n$ we predict $X^n$ using all other variables $\{X^m\}_{m\neq n}$ using a neural network $g_{\theta_n}:\Reals^{N-1}\to \Reals$.
We encourage $g_{\theta_n}$ to predict the expectation of $X^n$ using a sparse set of variables in $\{X^m\}_{m\neq n}$ by L$1$-regularizing the weights of its first layer.
In particular, calling $W^n$ the first layer weights of the network, $g_{\theta_n}$ is trained to minimize the objective
\vspace{-.2cm}
$$L_n(\theta_n)=E\left[X^n - g_{\theta_n}(X^m)_{m\neq n}\right]^2 + \lambda\sum_{i, j}|W^n_{i, j}|.$$

\vspace{-.5cm}
After training $g_{\theta_n}$ we can get a measure of the importance of $X^m$ in predicting $X^n$ as $\alpha_{X^n, X^m}=\sum_{i} (W^{n}_{i, m})^2$.
To get a moral graph in practice, we pick a threshold $\eta_3$ and connect $X^n$ and $X^m$ if $\alpha_{X^n, X^m} + \alpha_{X^m, X^n} > \eta_3$, that is, if $X^n$ is predicted to be in $\MB(X^m)$ or $X^m$ is predicted to be in $\MB(X^n)$.

It is possible that $g_{\theta_n}$ may erroneously ignore a variable $X^m$ which affects $X^n$ without changing its expectation.
To avoid this, as in Section~\ref{sec: diff objective}, we can train $g_{\theta_n}$ to predict the expectation of multiple statistics $T_1(X^n), \dots, T_C(X^n)$.
Again, in experiments we choose $C=2$ and $T_1, T_2$ as the first and second moments of $X$.

\subsection{Testing adjacency with two DATs} \label{sec: skel graph}
\vspace{-.1cm}
With knowledge of the moral graph, we can reduce the number of adjacency tests we must perform to learn the skeleton of $G$.
We can also reduce the search space for each test we perform:
the following proposition adapted from from \citet{Margaritis1999-vu} states that to test if two variables are adjacent in $G$,
instead of searching for a separating set in the set of all other variables, we can restrict our search to the Markov blanket of one of the variables.
\begin{proposition}\label{prop: reduce ad to moral}
    \textbf{\citep{Margaritis1999-vu}}
    (Proof in Appendix~\ref{app: proof reduce ad to moral})
    Assume $p$ is faithful. $X^n$ and $X^m$ are adjacent in $G$ if and only $X^n\notindep X^m\ |\ U$ for any $U\subset\MB(X^n)$.
    \vspace{-.1cm}
\end{proposition}
Now, to test if $X^n$ is adjacent to $X^m$ we can choose to search for a separating set $U\subset \MB(X^n)\setminus\{X^m\}$ or for a separating set ${U\subset \MB(X^m)\setminus\{X^n\}}$.
Classical testing methods solve the separating set selection problem by testing every subset of the search space;
    thus these methods saved a large amount of compute by learning the moral graph and then only solving the separating set selection problem corresponding to the smaller Markov boundary \cite{Margaritis1999-vu, Mokhtarian2020-lh}.

DAT does not scale poorly with the size of the search space, so DAT-Graph instead performs both tests and concludes that $X^n$ adjacent to $X^m$ if either of the tests states they are adjacent.
Performing two tests increases the reliability of the adjacency test:
there may be some subset $U\subset \MB(X^n)\setminus\{X^m\}$ that violates faithfulness with \mbox{$X^n\notindep X^m|U$} but such that $U$ is not a subset of $\MB(X^m)\setminus\{X^n\}$.
\begin{example}\label{ex: dg pro mb}
   \textbf{(Performing two tests increases reliability)} (Proof in Appendix~\ref{app: counterexamples})
    There are four jointly Gaussian random variables that are not faithful such that DAT-Graph recovers the correct graph.
\end{example}
Performing two tests also follows the idea of other methods that perform multiple tests that are redundant in the faithful case to be reliable when there are violations of faithfulness \cite{Spirtes2014-yq, Marx2021-wy}. 

In practice, we take the statistics of the two tests from Section~\ref{sec: diff objective}, $\hat L^n_{\mathrm{DAT}}(\psi^*)$ and $\hat L^m_{\mathrm{DAT}}(\psi^*)$, and decide that $X^n$ is adjacent to $X^m$ if either test statistic is large: $\hat L^n_{\mathrm{DAT}}(\psi^*)^2 + \hat L^m_{\mathrm{DAT}}(\psi^*)^2 > \eta_1$.
\subsection{Computational cost of DAT-Graph}\label{sec: datgraph scaling}
DAT-Graph has two steps.
In its first step it learns the moral graph.
This in principle scales quadratically with $N$, but in practice can scale to $N=10^4$ in hours \cite{Nazaret2023-yy}.
In the second step, DAT-Graph performs two adjacency tests for every edge in the moral graph.
If $s$ is the average number of parents in the graph $G$ then the Markov blanket of a variable $X^n$ often has approximately $O(s^2)$ edges.
Thus DAT-Graph needs to perform $O(Ns^2)$ tests which each involve a search over $O(s^2)$ variables.
Thus in principle, this step of DAT-Graph scales linearly with $N$ and is much faster on sparser graphs. 
\vspace{-.3cm}

\section{Learning from data with interventions}\label{sec: interventions}
\vspace{-.1cm}
In the purely observational setting, the graph $G$ can only be determined up to an equivalence class --
for some variables, we cannot distinguish which is the cause and which is the effect.
To learn cause and effect for these variables, we can collect data in an experiment where we intervene on a variable \cite{Hauser2012-vq}.
For example, we may knock down a gene in a cell.
We can then distinguish between cause and effect as intervening on a cause should affect the distribution of an effect, but not vise-versa.

In this section we extend DAT-Graph to learn from observational as well as intervention data of large and complex systems.
To do so, we model an intervention on a variable $X^n$ as a change in the dependence on its parents $h_n$ as defined in Eqn.~\ref{eq: functional}.
We can represent this by adding an extra argument to $h_n$
\begin{equation}\label{eq: functional w interventions}
    X^n = h_n\left(I, \epsilon_n, (X^m)_{m\in\Pa(X^n)}\right),
\end{equation}
where $I$ is a binary variable representing the presence of an intervention on $X^n$.
In intervention data we may have a number of intervention targets $X^{m_1}, \dots, X^{m_K}$ with intervention indicators $I^1, \dots, I^K$.
We assume we know which variables are intervened upon in each experiment, so $I^1, \dots, I^K$ are observed.

Just as Eqn.~\ref{eq: functional} produces a causal graph $G$ over the variables $X^1, \dots, X^N$,
Eqn.~\ref{eq: functional w interventions} produces a causal graph over the extended set of variables $X^1, \dots, X^N, I^1, \dots, I^K$ that has $G$ as an induced sub-graph.
We can therefore learn from intervention data by applying our method to learn the graph over the extended set of variables.
This approach is known as joint causal inference \cite{Mooij2020-di}.

Including the interventions in inference can help in orienting edges of the graph \cite{Mooij2020-di}.
We however note that if the targets of the intervention are known, then intervention data can also help learn the skeleton, reduce the number of tests we must perform, and can reduce the search space of our tests.
In Section~\ref{sec: intervention with known targets} we describe our method to learn from intervention data with known targets and show it help learn graphs more accurately in theory and in practice.

\vspace{-.1cm}
\section{Experiments} \label{sec: experiments}
\vspace{-.1cm}
Here we demonstrate that DAT-Graph can accurately and scalably learn graphs from data of large complex systems, with or without interventions, and on real and synthetic data.
We show that DAT-Graph performs particularly well on sparser graphs.
We also show that we can also combine the strengths of DAT-Graph and gradient-based model search methods in a hybrid method.

We measure the accuracy of inferred graphs with the Structural Hamming Distance (SHD) between the inferred skeleton and the true skeleton, and the SHD of the inferred directed graph and the true directed graph.
The SHD of the skeletons is the number of incorrectly inferred edges.
The SHD of the directed graphs is the SHD of the skeletons plus the number of incorrectly directed edges.

To implement DAT, we need to pick the distribution $f_m$ of the noise variables $(N_m)_m$ from Section~\ref{sec: discrete to diff search}.
To sample $N_m$ we first take $\tilde N_M\sim \mathrm{Laplace}$;
if $|\tilde N_M|\leq 1$ then we set $N_m = \frac 1 2 \tilde N_m$, otherwise, we scale $\tilde N_m$ to get thicker tails: $N_m = \frac 1 2 \mathrm{sgn}(\tilde N_m)|\tilde N_m|^{1.1}$.
In Appendix~\ref{sec: assump in exp} we show that this choice satisfies the assumptions of Thm.~\ref{Thm: main reliability}.
We perform all experiments on a single CPU and a single RTX 8000 GPU.
Other details of DAT are described in Appendix~\ref{app: method details}.
Experimental details are described in Appendix~\ref{app: exp details}.
\vspace{-.6cm}

\subsection{Learning from observational data}\label{sec: obs experiments}
\vspace{-.1cm}
\paragraph{Setup} First we demonstrate that DAT-Graph can accurately learn a causal graph from data of large complex systems in the purely observational setting.
To compare methods as fairly as possible, we simulate data of large complex systems as reported in the state of the art gradient-based model search method, SDCD, from \citet{Nazaret2023-yy}.
To do so, we generate an Erdős-Renyi random directed graph $G$ over $N$ variables with $s$ average parents.
We simulate complex relations between variables by using randomly initialized two layer neural networks with additive Gaussian noise as the causal relations between variables in Eqn.~\ref{eq: functional}.
We then generate $10000$ datapoints from this model.

\citet{Nazaret2023-yy} performed a thorough investigation of the scalability of different graph learning algorithms.
They showed that, on this data, existing algorithms -- other than their algorithm, SDCD -- do not scale to more than $100$ variables or only achieve trivial accuracy.
We therefore use SDCD as our baseline.
SDCD is a gradient-based model search method with a similar first step to DAT-Graph -- they both first learn the moral graph using a variable selection procedure.
Then, when SDCD is learning the causal graph, it masks edges that are not in the learned moral graph.
\vspace{-.3cm}

\paragraph{Large graphs} 
In Fig.~\ref{fig:observation} we plot the SHD in the inferred skeleton and inferred directed graph on data with $s=4$ and various values of $N$ as in \citet{Nazaret2023-yy}.
We note DAT-Graph infers the graph just as accurately as the state of the art method SDCD at all $N$.
We also note that DAT-Graph becomes relatively more accurate as $N$ increases, in this case beating the state of the art model SDCD by a substantial margin for large $N$.
In Fig.~\ref{fig:observation sf} and \ref{fig:observation lin} in the Appendix we also show a similar result when the graph is generated from a scale-free distribution or with linear relations.
In App.\ref{sec: small N} we show that DAT-Graph achieves state of the art accuracy among small-scale methods as well.
\vspace{-.2cm}

\begin{figure}
    \centering
    \includegraphics[trim={0cm 0.cm 0.0cm 0.cm}, width=0.99\columnwidth]{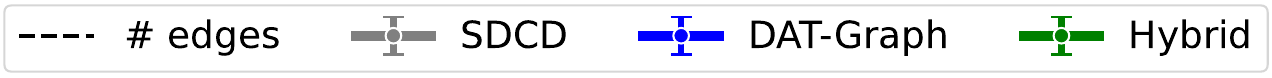}
    \subfigure[][\label{fig: obs skel} Skeleton]{\includegraphics[trim={1.1cm 0.6cm 0.0cm 0.6cm}, width=0.44\columnwidth]{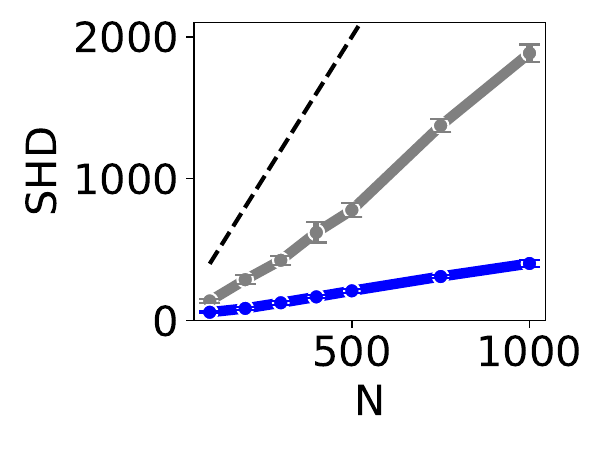}}
    \subfigure[][\label{fig: obs graph} Directed graph\hspace{-0.3cm}]{\includegraphics[trim={0.5cm 0.6cm 0.6cm 0.6cm}, width=0.44\columnwidth]{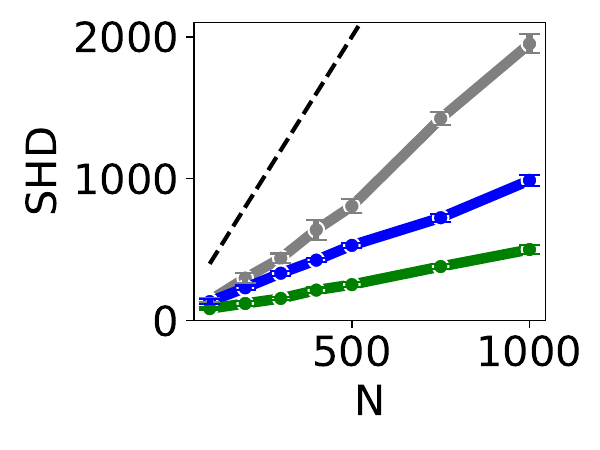}}
    \vspace{-.3cm}
    \caption{\textbf{DAT-Graph learns large graphs accurately.} 
    We plot the mean error (SHD) and standard error against the size of the graph ($N$) across 3 replicates.} \label{fig:observation}
    \vspace{-.5cm}
\end{figure}

\paragraph{Hybrid model}
Next we demonstrate that DAT-Graph and model search methods have complimentary strengths that can be combined in a hybrid model.
The SHD of SDCD's inferred skeleton is similar to the SHD of its directed graph, meaning it often picks the correct direction for arrows in the graph.
This is not the case for DAT-Graph, which infers a very accurate skeleton but makes more errors when orienting edges.
To combine the advantages of these methods, we create a hybrid method by first learning a skeleton using DAT-Graph and then learning the directions of the edges using SDCD.
Fig.~\ref{fig: obs graph} shows that this hybrid method performs substantially better than both methods at scale.
\vspace{-.2cm}

\paragraph{Scaling}
Our results also demonstrate that DAT-Graph can theoretically scale to learn from very large datasets.
SDCD is incredibly scalable -- it infers a graph of 1000 nodes in roughly 25 minutes.
DAT-Graph is not as scalable as SDCD however it scales to large systems in reasonable time -- it infers a graph of 1000 nodes in roughly 4 and a half hours.
As well, in Fig.~\ref{fig: time n} we show that compute time for DAT-Graph scales roughly linearly with $N$;
this is expected as, keeping $s$ fixed, the number of adjacency tests in DAT-Graph scales linearly with $N$ as discussed in Section~\ref{sec: skel graph}.
If this linear trend continuous, DAT-Graph could learn from the largest transcriptomics datasets, which could include a variable for all 20000 genes in a human, in a few days.

\vspace{-.2cm}
\paragraph{Sparsity}
Graphs of real data are likely to be sparse -- each variable is caused by few others -- and we would like to take advantage of this sparsity to learn a more accurate graph.
In principle, both SDCD and DAT-Graph take advantage of sparsity to exclude more edges when learning the moral graph.
In Fig.~\ref{fig:sparsity} we test how well each method takes advantage of sparsity in practice by plotting the error in the inferred graphs for datasets with $N=200$ and various average numbers of parents $s$.
All methods perform more accurate inference on sparser graphs, but DAT-Graph benefits from sparsity much more than SDCD -- when $s=2$ the mean SHD of the graph inferred by SDCD is 112, while that of DAT-Graph and the hybrid method are 45 and 29 respectively.
In Fig.~\ref{fig: time s} in the Appendix we also show that DAT-Graph also requires less compute to learn sparser graphs.
On the other hand, SDCD makes more accurate predictions on denser graphs, a possible advantage gradient-based learning methods.
In Fig.~\ref{fig:observation s6} in the Appendix we confirm that the conclusion of Fig.~\ref{fig: obs graph} do not change when the graph is dense -- DAT-Graph and the hybrid method make more accurate predictions as the graph gets larger.
\vspace{-.2cm}

\begin{figure}
    \centering
    \subfigure[][\label{fig: sparse skel} Skeleton]{\includegraphics[trim={1.1cm 0.6cm 0.0cm 0.6cm}, width=0.4\columnwidth]{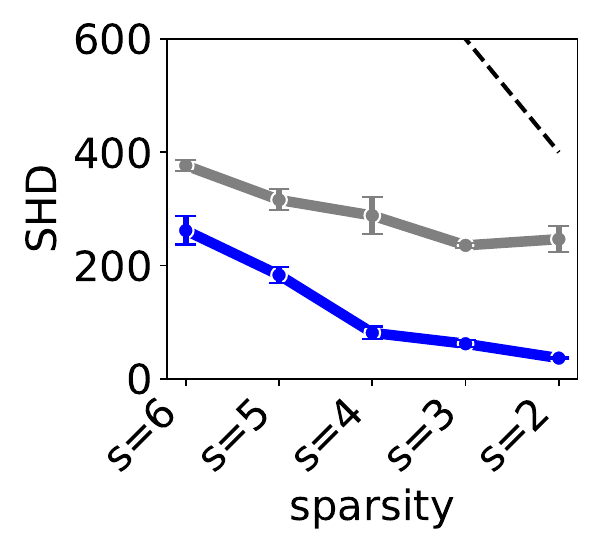}}
    \subfigure[][\label{fig: sparse graph} Directed graph\hspace{-0.3cm}]{\includegraphics[trim={0.5cm 0.6cm 0.6cm 0.6cm}, width=0.4\columnwidth]{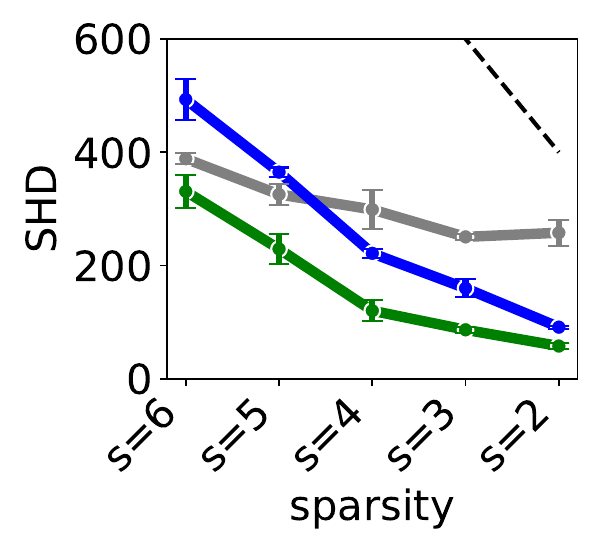}}
    \vspace{-.3cm}
    \caption{\textbf{DAT-Graph learns sparser graphs more accurately.}
    We plot the mean and standard error against the sparsity of the graph across 3 replicates.
    The legend is the same as that of Fig.~\ref{fig:observation}.
    \vspace{-0.5cm}} \label{fig:sparsity}
\end{figure}

\subsection{Learning from intervention data}
Given more intervention data, we expect to be able to learn a more accurate graph.
To see if DAT-Graph efficiently uses intervention data to learn more accurate graphs, we simulate data similar to the setup above but intervene on certain variables.
Intervened variables are drawn from a Gaussian distribution with standard deviation $0.1$.
We vary the fraction of variables that are intervened upon.
We simulate $10000$ datapoints from the observational distribution and for every intervened variable we sample another $500$ datapoints where that variable is intervened upon.
\citet{Nazaret2023-yy} demonstrated that SDCD learns from intervention data substantially better than other methods.
Thus we use SDCD as our baseline.
\vspace{-.1cm}

In Fig.~\ref{fig:intervention} we plot the SHD of the skeleton and directed graph when $N=100$, $s=4$ for datasets with various fractions of variables intervened.
We see that all methods efficiently use intervention data -- an increasing amount of intervention data makes predictions more accurate.
\vspace{-.2cm}

\begin{figure}
    \centering
    \includegraphics[trim={0cm 0.cm 0.0cm 0.cm}, width=0.78\columnwidth]{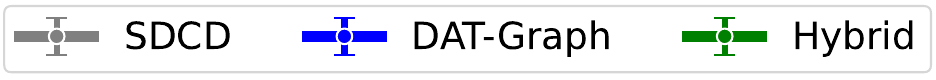}
    \subfigure[][\label{fig: int skel} Skeleton]{\includegraphics[trim={1.2cm 0.6cm 0.0cm 0.6cm}, width=0.4\columnwidth]{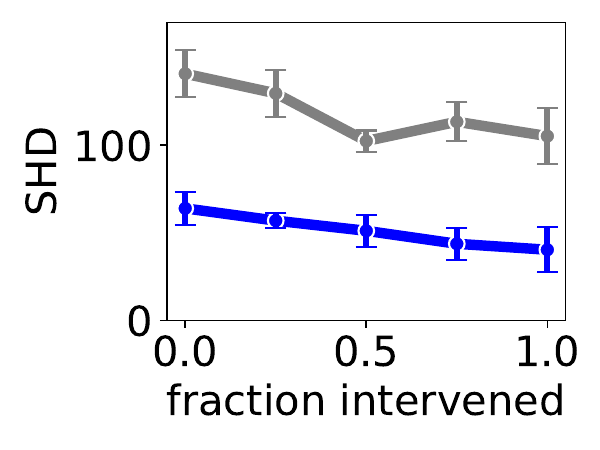}}
    \subfigure[][\label{fig: int graph} Directed graph\hspace{-0.3cm}]{\includegraphics[trim={0.6cm 0.6cm 0.6cm 0.6cm}, width=0.4\columnwidth]{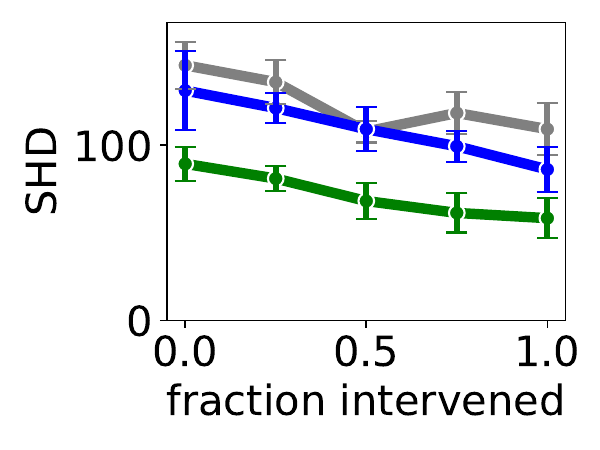}}
    \vspace{-.3cm}
    \caption{
    \textbf{DAT-Graph learns more accurate graphs when given intervention data.} 
    We plot the mean and standard error against the number of variables with interventions across 5 replicates.
    \vspace{-.5cm}} \label{fig:intervention}
\end{figure}

\subsection{Sensitivity to model choices}
\vspace{-.1cm}
In Section~\ref{sec: discrete to diff search} and \ref{sec: skel graph} we justified our choices of representation $r_\psi$ and testing each edge twice in theory.
In Section~\ref{sec: ablations} we perform ablations that show that these choices also substantially increase the accuracy of DAT-Graph in practice.

In Fig.~\ref{fig:robust nn} and \ref{fig:robust eta} in the appendix we show that DAT-Graph is robust to neural network and threshold hyperparameter choices.
In Fig.~\ref{fig:robust c} we investigate the effect of adding more statistics to our estimate of variance explained in Section~\ref{sec: diff objective}; we see including the first two moments does better than including only the first moment or the first three, likely because two moments best balances the ability for us to detect dependencies with the variance of the estimator.

\vspace{-.2cm}
\subsection{Predicting interventions on RNA sequencing data}
\vspace{-.1cm}
Learning which variables are causes of which others in principle allows us to better predict the effects of interventions.
In this section we investigate whether the graph learned by DAT-Graph can be useful for predicting interventions on large complex systems. 
Here we learn from a single-cell RNA sequencing experiment of cancer that is resistant to immunotherapy \cite{Frangieh2021-nd} to predict the effects of gene knockdowns.
Good prediction can tell us about the mechanisms of resistance and suggest targets for treatment.

Each variable $X^n$ is the normalized transcript count of gene $n$ and each data point is the transcript counts for every gene in a cell.
Interventions are CRISPR gene knockdowns; there can be multiple interventions per cell.
We preprocessed this data as in \citet{Lopez2022-iz}.
We first split the data into the three cell populations studied in \citet{Frangieh2021-nd} --- control, co-culture, and IFN-$\gamma$-treated cells.
We then filtered to predict on the $N=1000$ most variable genes.
We split each dataset into a training set and a test set containing interventions that are not in the training set.

We infer a graph from each training set using DAT-Graph and evaluate whether restricting the graph search of SDCD with these graphs can improve prediction.
In Fig.~\ref{fig: rna performance} we show that 
SDCD's prediction improves substantially when graph search is restricted to edges learned in the skeleton of DAT-Graph. 
In Appendix~\ref{sec: rna performance second hybrid} we show that this improvement is not an artefact of training SDCD.
The hybrid model also outperforms another gradient-based model search method, DCDFG, which was built to learn on large-scale RNA sequencing data \cite{Lopez2022-iz}.

\begin{figure}
    \centering
    \includegraphics[trim={0cm 0.0cm 0.0cm 0.cm}, width=0.9\columnwidth]{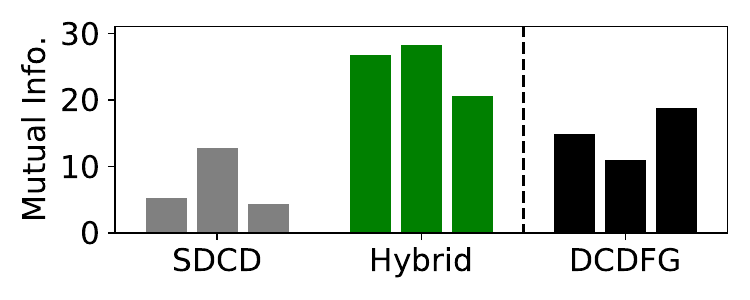}
    \vspace{-.5cm}
    \caption{
    \textbf{DAT-Graph helps predict the effects of unseen interventions on RNA sequencing data.} 
    We plot the learned mutual information --- the difference between the mean log likelihood of a model and the mean log likelihood of a trivial model with an empty graph on the test set --- of models for the datasets ``control'', ``IFN-$\gamma$'', and ``co-culture''. 
    We compare two graph search algorithms --- SDCD and DCDFG --- to our method which is a hybrid of DAT-Graph and SDCD.
    \vspace{-.6cm}} \label{fig: rna performance}
\end{figure}

\section{Conclusion} \label{sec: conclusion}
We have developed DAT-Graph to scalably and reliably learn cause and effect from data of large and complex systems.

There are many exciting directions for future work.
While DAT-Graph accurately learns sparse graphs, it can be less accurate than model selection based methods when learning dense graphs;
    future work could build hybrid models that have the strengths of both approaches.
DAT-Graph also follows a trend of modern multi-step graph learning methods that come with more hyperparameters \cite{Nazaret2023-yy, Lopez2022-iz};
    future work could reduce the number of hyperparameters by combining steps in these methods, for example by reusing neural networks between steps.

In this work we have made the assumption that all variables are observed and there are no cycles in $G$.
However, there are large and complex systems where these assumptions are not to likely to hold \cite{Sethuraman2023-ze, Lorch2023-qe}.
Methods have been developed to test for confounding and cycles by solving the separating set selection problem, but only at small scale \cite{Spirtes2001-sj, Richardson1996-yb}.
Future work could scale these methods using DAT.

In this work we have considered learning the entire graph $G$.
However, often one is only interested in learning cause and effect for a particular node.
Unfortunately model search methods must learn the whole graph.
Future work may apply DAT to only learn the skeleton around this node.
In addition to saving compute, learning the skeleton without necessarily orienting edges is a procedure that is robust to confounding or cycles \cite{M_Mooij2020-um}.

\section*{Impact Statement}
DAT-Graph allows accurate inference of cause and effect in systems of many variables that interact in complex ways.
Such systems appear in genetics, microbiology, and health.
Understanding cause and effects in these settings can help us understand the mechanisms of disease, and help us build better interventions and treatments.
On the other hand, causal conclusions about genetics and phenotype can be used to justify harmful policies.

\bibliography{bibliography}
\bibliographystyle{icml2024}

\newpage
\appendix
\onecolumn

\section{Review of graph learning methods}\label{app: related work}

We attempt to learn graphs with no unobserved latent variables or cycles.
There is an extensive history of algorithms to learn such graphs from data~\citep{Vowels2022-gr}. 
Methods can learn a graph by testing for conditional independence \citep{Spirtes1993-ty, Spirtes2001-sj, Margaritis1999-vu}, optimizing an objective \citep{Chickering2002-rx, Van_de_Geer2013-ak, Raskutti2018-fo}, performing independent component analysis \citep{Shimizu2006-hj, Shimizu2011-oq, Reizinger2022-ep}, looking for nonlinearities \citep{Rolland2022-ea, Montagna2023-wm, Montagna2023-wv}, and more \citep{Immer2023-nx, Gao2020-av, Reisach2021-sd}.
Unfortunately, \citet{Nazaret2023-yy} showed that almost all methods made strong assumptions on the form of the cause-effect relationships between variables or can not scale to more than 100 variables.
Our work aims to perform flexible inference at large scale.

Recently a class of optimization-based graph learning methods were able to scale to learn from data of large complex systems by searching through the space of graphs at the same time as training flexible neural networks to model causal relationships \citep{Zheng2018-dc, Lachapelle2019-jl, Zheng2020-wl, Bello2022-xz, Nazaret2023-yy}.
These methods can also learn from data with interventions \citep{Brouillard2020-wo, Nazaret2023-yy}.
Unfortunately, the model search can involve an unstable optimization problem over the enormous space of all graphs and has been shown to be unreliable in some settings \citep{Wei2020-ju, Nazaret2023-yy, Deng2023-ix}.
\citet{Nazaret2023-yy} recently substantially improved the accuracy of these methods by first learning the moral graph to shrink the model search space.
”Hybridizing” model search with some conditional independence testing is also the strategy of the most accurate classical graph learning methods \citep{Tsamardinos2006-zk, Buhlmann2014-do}.
Our work aims to further shrink the model search space, or learn the graph entirely, by testing for conditional independence at scale.

\section{Details of the method}\label{app: method details}

\subsection{Orienting the edges of the skeleton}\label{sec: full graph}
Once we have learned the skeleton, we finally need to decide the direction of its arrows.
The equivalence class of a graph is determined by its skeleton and v-structures -- variables $X^n, X^k, X^m$ such that $X^n$ and $X^m$ are not adjacent and $X^n\rightarrow X^k\leftarrow X^m$ in $G$ \cite{Verma2022-cs}.
For all v-structures $X^n, X^k, X^m$ we have that $X^n$ and $X^m$ are spouses, so they are adjacent in the moral graph but not the skeleton.
As well, if $X^n$ and $X^m$ are spouses and $X^n - X^k - X^m$ in the skeleton, then if $p$ is faithful to $G$, $X^k\not \in \sepset(X^n, X^m)$ if and only if $X^n\rightarrow X^k\leftarrow X^m$ in $G$.

To infer v-structures, we first pick a threshold $\eta_2>0$ and look for any triplet in the inferred skeleton $X^n - X^k - X^m$ such that $X^n$ and $X^m$ are adjacent in the inferred moral graph but not the inferred skeleton.
We then decide if $X^k\in\sepset(X^n, X^m)$ using the two learned parameters $\psi^n$ and $\psi^m$ from the two tests between $X^n$ and $X^m$.
We label the triplet a v-structure if $(\psi_k^n)^2 + (\psi_k^m)^2<\eta_2$, that is if $k$ is not in the separating set for either of the tests we did for the pair $X^n, X^m$, otherwise we label it not a v-structure.
After we have labelled all of the v-structures, we can apply Meek's rules to orient many of the remaining edges as in the PC algorithm \cite{Spirtes1993-ty}.
We then use some heuristics described in Appendix~\ref{app: equiv to graph} to extract a single graph from the equivalence class.

\subsection{Learning the graph with interventions with known targets}\label{sec: intervention with known targets}
For clarity, assume $p$ is faithful in this section.
We first learn the moral graph over all variables $X^1, \dots, X^N, I^1, \dots, I^K$ with a minor modification.
Predicting $\MB(I^k)$ is unreliable in practice so we do not use it to build the moral graph.
Instead we just connect $X^n$ and $I^k$ in the moral graph if $I^k$ is predicted to be in $\MB(X^n)$ -- we pick a threshold $\eta_4>0$ and connect the two variables if $\alpha_{X^n, I^k}>\eta_4$.

Next we learn the skeleton over the variables $X^1, \dots, X^N, I^1, \dots, I^K$.
We reduce the number of tests we perform with three techniques:

1) We do not test adjacencies of intervention variables $I^k$.

2) We learn the parents of intervened variables during the moral graph learning step.
If $X^n$ is adjacent to an intervention $I^k$ with target $X^m\neq X^n$ in the moral graph then, since $X^n$ cannot be the parent or child of $I^k$ it must be its spouse -- $X^n\in\Pa_G(X^m)$.
Thus if $X^n$ is adjacent to $I^k$ in the moral graph then we label $X^n$ a parent of $X^m$.
We do not need to test the adjacency of $X^n$ and $X^m$.
\footnote{In principle we could orient every edge connected to the target of an intervention.
Say we determine that $X^n$ is adjacent to $X^m$ in $G$ where $X^m$ is the target of an intervention $I^k$.
If $I^k$ is not adjacent to $X^n$ in the moral graph then $X^n$ must be the child of $X^m$.
In practice however, we have seen that this strategy is unreliable in the absence of a large amount of intervention data.}

3) We use the direction of edges learned in the moral graph step to shrink the search space of the adjacency test.
The second statement in Prop.~\ref{app: proof reduce ad to moral} states we can always include parents of $X^n$ and we can always exclude sinks -- variables with no children -- in the separating set.
Thus, for $X^k\in\MB(X^n)$, if $X^k$ is a parent of $X^n$ we fix $\psi^n_k=1$;
and if we have determined that every node that $X^k$ is adjacent to in the moral graph is its parent then we fix $\psi^n_k=0$.

Once we have learned the skeleton of the variables $X^1, \dots, X^N$ and oriented some edges, we orient the remaining edges in $X^1, \dots, X^N$ just as in Section~\ref{sec: full graph}.

In Table~\ref{tab: jci ablation performance} we see that using our method above allows DAT-Graph to learn skeletons more accurately in practice than by simply including the intervention variables as nodes in the graph (Naive JCI).

\begin{table}[H]
\caption{\textbf{DAT-Graph accurately learns skeletons using intervention data.} 
    Accuracy of predicting adjacencies in the experiment in Fig.~\ref{fig:intervention} with $N=200$ and 50\% of variables intervened on.
	\label{tab: jci ablation performance}} 
  \setlength\tabcolsep{3pt}
  \begin{center}
\begin{small}
\begin{sc}
\begin{tabular}{ c c c c } 
	\toprule
	{Model} & {Errors (skeleton SHD)}\\ \midrule
    DAT-Graph & \textbf{51±9}\\ \midrule
	Naive JCI & 69±9 \\
\end{tabular}
\end{sc}
\end{small}
\end{center}
\end{table}

\subsection{Getting a graph from an equivalence class}\label{app: equiv to graph}
After we are done applying Meek's rules, there may be a small number of non-oriented edges -- we may have only identified the graph up to an equivalence class.
To get a single graph, we iteratively randomly orient a randomly chosen unoriented edge and re-apply Meek's rules until all edges are oriented.

There may also be cycles in the graph we learn, $\hat G$, due to disagreeing tests.
There are sophisticated methods to learn a graph in the case that tests disagree \cite{Triantafillou2015-gu}.
However, in our experiments there are usually only a small number of cycles in $\hat G$, so we take a simple approach to remove cycles.
First we calculate the matrix $(\hat G+I)^N-I$.
The $n$-th entry on the diagonal of this matrix is $0$ if and only if the $n$-th node is not in a cycle \cite{Zheng2018-dc}.
While there are non-zero entries on the diagonal, we pick the node $n$ that maximizes the value $((\hat G+I)^N-I)_{n, n}$ and remove an edge connected to this node that maximally reduces $\mathrm{trace}((\hat G-I)^N-I)$.

\subsection{Hyperparameters}
We have four threshold hyperparameters: one for deciding the edges of the moral graph $\eta_3$, one for deciding edges in the skeleton $\eta_1$, one for deciding v-structures $\eta_2$, and one for deciding edges in the moral graph connected to intervention variables $\eta_4$.
We choose $\eta_3=8\times 10^{-3}$, $\eta_1=10^{-4}$, $\eta_2=0.2$, $\eta_4=10^{-3}$.
Theorem~\ref{Thm: main reliability} suggests that $\eta_2$ should be a larger number.
We noticed however that a smaller value of $\eta_2$ resulted in more accurate graph recovery.
We discuss why this might be in Appendix.~\ref{sec: thm disscussion}.

Each variable selection problem in inferring the moral graph has a sparsity parameter $\lambda_n$.
We noticed when that $\sigma^2_n=E[X^n-g_{\theta_n}(X^m)_{m\neq n}]^2$ could vary drastically from node to node.
This caused the influence of the sparsity penalty to vary from node to node, making it challenging to get accurate graph recovery with a single threshold $\eta_3$.
To address this issue, we found the scaling the sparsity parameter $\lambda_n$ by $\sigma^2_n$ improved recovery of the moral graph.
Thus for the $n$-th variable we use a sparsity penalty of $0.01\times \sigma^2_n$ where $\sigma^2_n$ is estimated from the current minibatch.

We used batch sizes of size 256 in all cases.
To train the neural networks to predict the moral graph, we used the Adam optimizer with parameters $\beta_1, \beta_2=0.9, 0.999$ and learning rate $10^{-4}$ and trained for 30000 minibatches.
We train the models for all nodes in parallel on a GPU.

To train the networks to predict the skeleton we trained for 10000 minibatches and took alternating steps to update $\{\theta_1, \theta_2\}$ and $\psi$.
For $\{\theta_1, \theta_2\}$ we used the Adam optimizer with parameters $\beta_1, \beta_2=0.9, 0.999$ and learning rate $3\times 10^{-4}$ while for $\psi$ we used $\beta_1, \beta_2=0.9, 0.9$ and a learning rate of $3\times 10^{-4}$.
We train models for all tests in parallel on a GPU.

To predict the moral graph, we used 3 layer neural networks with 200 hidden units.
We used a ReLU activation and included dropout and batchnorm between layers.
We used a dropout probability of $0.1$ between the first and second layer and a probability of $0.5$ between the second and third.
To predict the skeleton, we used 3 layer neural networks with 100 hidden units.
We again used a ReLU activation and included dropout and batchnorm between layers.

\subsection{Other details}
Before learning the graph we normalize all variables in the data to have mean $0$ and standard deviation $1$.

We parameterize $\psi_n\in[0, 1]$ as $\psi_n=\mathrm{sigmoid}(\gamma_n)$ for $\gamma_n\in (0, \infty)$. $\gamma_n$ is the parameter we optimize by gradient descent.

When testing the adjacency between variables $X^n, X^m$, we have the choice of using either $E\mathrm{VarExplained}(X^n;X^m|\MB(X^n))$ or $E\mathrm{VarExplained}(X^m;X^n|\MB(X^n))$ as our measure for conditional independence.
We use the later in experiments as we found it to make more accurate decisions.

We use combined samples from 5000 mini-batches to calculate $\hat L_{\mathrm{DAT}}(\psi)$.

\section{Experimental details}\label{app: exp details}

\subsection{Data simulation details}
We simulate observational data just as in \citet{Nazaret2023-yy}.
We use code from \url{https://github.com/azizilab/sdcd} under an MIT licence.
Briefly, we generate a random undirected graph $\tilde G$ and a random permutation of $\{1, \dots, N\}$, $\pi$.
Then we have an edge $n\rightarrow m$ in the graph $G$ if $\pi(n)>\pi(m)$ and $n$ and $m$ are connected in $\tilde G$.
Next we model the functions 
$$X^n \sim \tilde h_n\left(\frac{X^m-E[X^m]}{\mathrm{Std}(X^m)}\right)_{m\in\Pa_G(n)}+N(0, 1)$$
where $\tilde h_n$ is a randomly initialized two layer neural network with ReLU activations and 100 hidden units.
If $\Pa_G(X^n)=\emptyset$, we set $\tilde h_n=0$.
When simulating data with linear relations, we replace $\tilde h_n$ with a linear model with weights drawn from $N(0, 1)$.

In our experiments with interventions, we assume that only a single variable is intervened on in each data point.
If a variable $X^n$ is intervened on then $X^n\sim 0.1 * N(0, 1).$

The SHD between two undirected graphs $G_1, G_2$ is the sum of the number of edges in $G_1$ but not $G_2$ and the number of edges in $G_2$ but not $G_1$.
The SHD between two directed graphs is the sum of the SHD of their skeletons plus the number of edges pointed in the wrong direction.

We simulated Erdős-Renyi and scale free random graphs using code from \citet{Montagna2023-hw}.

\subsection{Ablation experiments}\label{sec: ablations}
We performed ablations that demonstrate the benefits of our modelling decisions.
We generated observational data as in Section~\ref{sec: experiments} with $N=200$ and $s=4$;
results are shown in Table~\ref{tab: ablation performance} with mean skeleton SHD (defined in Section~\ref{sec: experiments}) and standard deviation across 3 replicates.

To demonstrate that our choice of representation provides a more accurate answer to the separating set selection problem we performed ablations where we replaced $r_\psi$ with a 3 layer neural network $r_\psi:\mathbb R^M\to\mathbb R$ with 200 hidden units (\textbf{Neural net $r_\psi$}).
We used a ReLU activation and included dropout and batchnorm between layers.
We used a dropout probability of $0.1$ between the first and second layer and a probability of $0.5$ between the second and third.
We then optimize $\psi$, the parameters of the neural network.
We also performed ablations where we used noise distributions $(f_m)_m$ with thin tails -- $f_m$ were Gaussian densities -- rather than the thick tailed distribution described in section~\ref{sec: experiments} (\textbf{Gaussian $(f_m)_m$}).

To demonstrate that testing twice as discussed in Section~\ref{sec: skel graph} makes our method more reliable, we perform an ablation where we decide if $X$ is adjacent to $Y$ by randomly testing one of $X\indep Y\ |\ U$ for some $U\subset\ \MB(X)\setminus\{Y\}$ or $X\indep Y\ |\ U$ for some $U\subset\ \MB(Y)\setminus\{X\}$ (\textbf{Only one test}).
We halve $\eta_1$ for this ablation.

In addition to the results shown in Table~\ref{tab: ablation performance} we also demonstrate that solving the separating set selection problem is informative for learning the graph: we perform ablations where we do not learn the parameters $\psi$: we either test if two nodes are adjacent in the skeleton by testing their marginal independence ($\psi=0$) or their conditional independence ($\psi=1$).
We show the results in Table~\ref{tab: ablation performance}.

\begin{table}[H]
\caption{\textbf{Ablations justify the choices in DAT-Graph in practice.} 
    Accuracy of predicting adjacencies in a graph with $N=200$ and $800$ edges with standard deviations across three replicates.
    Details in Appendix~\ref{sec: ablations}.
	\label{tab: ablation performance}} 
  \setlength\tabcolsep{3pt}
  \begin{center}
\begin{small}
\begin{sc}
\begin{tabular}{ c c c c } 
	\toprule
	{Model} & {Errors (skeleton SHD)}\\ \midrule
    DAT-Graph & \textbf{81±9}\\ \midrule
	Neural net $r_\psi$ & 173±11 \\
    Gaussian $(f_m)_m$ & 98±12 \\\midrule
	Only one test & 177±16 \\ \midrule
 Test marginal $\psi=0$ & 125±19 \\
    Test conditional $\psi=1$ & 308±14 \\\bottomrule
\end{tabular}
\end{sc}
\end{small}
\end{center}
\end{table}

\subsection{Baseline methods}\label{sec: baseline methods}
We implemented SDCD using the code from \url{https://github.com/azizilab/sdcd}.
We used the hyperparameters described in \citet{Nazaret2023-yy}.
In \citet{Nazaret2023-yy} SDCD was trained on 2000 epochs on 10000 datapoints.
When we added intervention data or trained on real RNA sequencing data, we scaled the number of epochs with the dataset size proportionally.
We also scaled the parameter \texttt{gamma\_increment} which controls the increment of the acyclicity penalty per epoch.

We also compare to a number of less scalable and flexible methods at small scale in Section~\ref{sec: small N}.
    \textbf{PC \cite{Spirtes1993-ty}:} a classical method to learn the graph by looking for conditional independence relations; we implement this algorithm using a kernel test for conditional independence.
    We used code from \url{https://pywhy.org/dodiscover/} under an MIT licence with kernel threshold $0.05$.
    \textbf{GES \cite{Chickering2002-rx}:} a classical model selection procedure that performs the model search with greedy perturbations to the graph; it assumes all variables are jointly Gaussian.
    We implemented this using the code from \citet{Hauser2012-vq}.
    \textbf{CAM \cite{Buhlmann2014-do}:} a classical hybrid graph learning method that assumes that causal interactions are additive.
    We used code from \url{https://pywhy.org/dodiscover/} under an MIT licence.
    \textbf{NoGAM \cite{Montagna2023-wv}:} a recent method that learns conditional independence relations and infers a graph by assuming additive noise;
    it recently performed best in an array of small scale settings against other small scale graph learning methods \citep{Montagna2023-hw}.
    We used code from \url{https://pywhy.org/dodiscover/} under an MIT licence.

\subsection{RNA sequencing data}
We learned on single cell RNA sequencing data from a study of immunotherapy resistant cancer \citet{Frangieh2021-nd}.
These cells had various genes perturbed by CRISPR knockdowns.
We preprocessed this data as in \citet{Lopez2022-iz} using code from \url{https://github.com/Genentech/dcdfg} under an Apache-2.0 licence.
The preprocessed dataset included between 57523 and 87436 cells and measurements of $N=1000$ genes.
Data include observational and intervention samples.
We created a test set by selecting 20\% of intervention targets and holding out samples of those interventions.
The test set had between $6984$ and $11993$ samples.

With default settings, SDCD predicts the conditional variance of variables with a neural network.
We noticed that on this dataset, SDCD makes worse-than-trivial predictions with this setting (unless hybridized with the skeleton learned using DAT-Graph).
Thus we used the setting \texttt{model\_variance\_flavor='parameter'}.
We also used \texttt{finetune=True} when training to get a valid likelihood on the test set.

We implemented the \texttt{MLPGaussialModel} from DCDFG using code from \url{https://github.com/Genentech/dcdfg} with hyperparameters that were optimized according to \citet{Lopez2022-iz} on the data, that is, $m=20$, $\lambda =10^{-3}$, trace exponential penalty.

Both SDCD and DCDFG model the data as coming from a Gaussian additive model.
A trivial prediction for such a model is that all variables are generated from iid Gaussians.
The mean negative log likelihood of this trivial model on a test set can be calculated as 
$$\sum_{n=1}^N\frac 1 2 \log(2\pi \mathrm{Var}_{\mathrm{train}}(X^n)) + \frac 1 2\frac{E_{\mathrm{test}}\left[X^n-E_{\mathrm{train}}X^n\right]^2}{\mathrm{Var}_{\mathrm{train}}(X^n)}.$$

To build the hybrid method, we reasoned that the same hyperparameters that are optimal for graph recovery as measured by SHD may not be optimal for intervention prediction.
It may be the case for example that leaving out a causal arrow in the graph may harm prediction much more than including spuriously inferred edges.
Thus we used the same hyperparameters as above but picked $\eta_3\in\{0.001, 0.003, 0.005, 0.008\}$ based on what minimized the fit on the training data on the ``control'' dataset set.
We found $\eta_3=0.001$ lead to the best fit on the training data and used this value for experiments in the main text.

\section{Further experimental results}
\subsection{Appendix to main text figures}

\begin{figure}[H]
    \centering
    \subfigure[][\label{fig: obs sf skel} Skeleton]{\includegraphics[width=0.2\columnwidth]{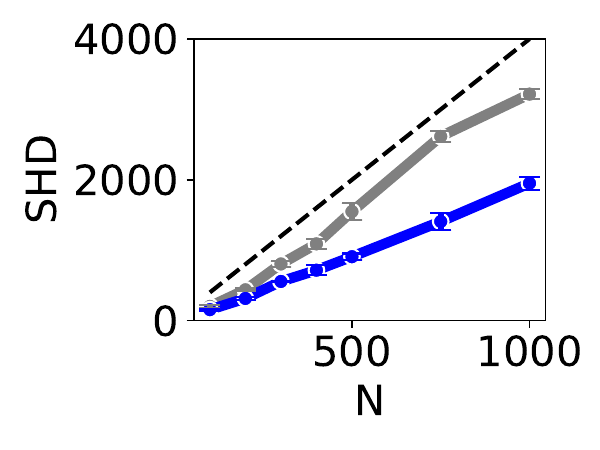}}
    \subfigure[][\label{fig: obs sf graph} Directed graph]{\includegraphics[width=0.2\columnwidth]{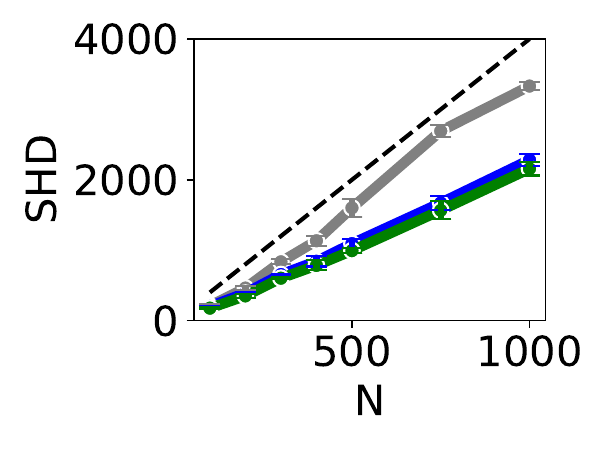}}
    \caption{\textbf{DAT-Graph can accurately learn from data of a scale-free random graph.} We perform the experiment in Fig.~\ref{fig:observation} with scale-free random graphs.
    We plot the mean SHD and standard deviation across 3 replicates.
    The legend is the same as that of Fig.~\ref{fig:observation}.} \label{fig:observation sf}
\end{figure}

\begin{figure}[H]
    \centering
    \subfigure[][\label{fig: obs lin skel} Skeleton]{\includegraphics[width=0.2\columnwidth]{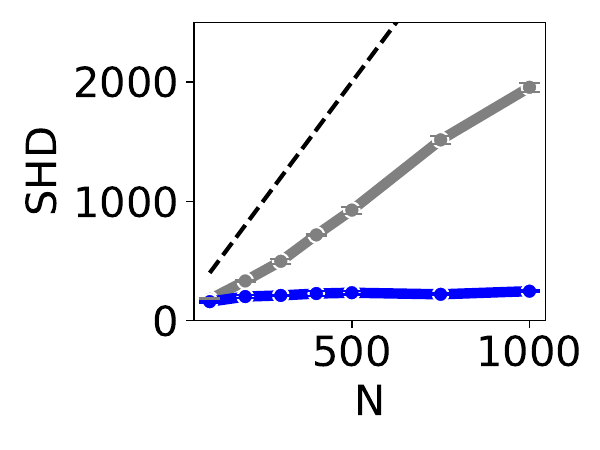}}
    \subfigure[][\label{fig: obs lin graph} Directed graph]{\includegraphics[width=0.2\columnwidth]{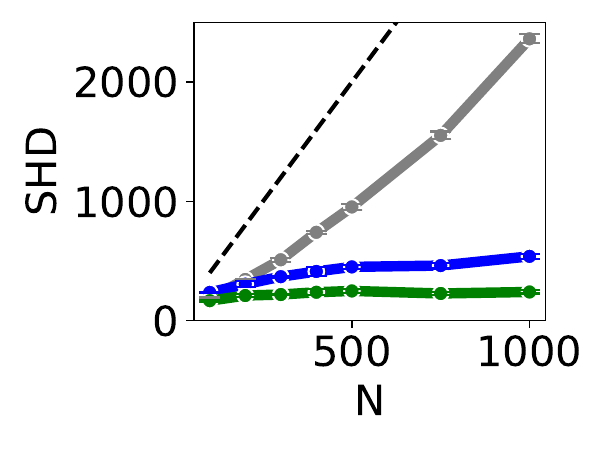}}
    \caption{\textbf{DAT-Graph can accurately learn from data with linear relations.} We perform the experiment in Fig.~\ref{fig:observation} with linear relations between variable.
    We plot the mean SHD and standard deviation across 3 replicates.
    The legend is the same as that of Fig.~\ref{fig:observation}.} \label{fig:observation lin}
\end{figure}

\begin{figure}[H]
    \centering
    \subfigure[][\label{fig: time n} Scaling with $N$]{\includegraphics[width=0.2\columnwidth]{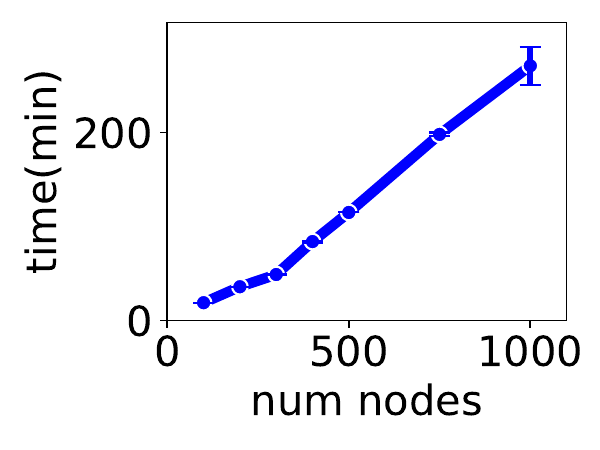}}
    \subfigure[][\label{fig: time s} Scaling with $s$]{\includegraphics[width=0.17\columnwidth]{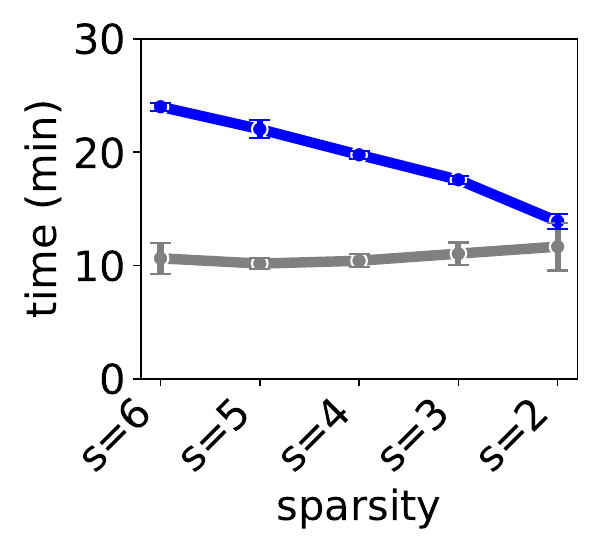}}
    \caption{\textbf{DAT-Graph scales linearly with $N$ and is more efficient on sparser graphs.}
    A) For $s=4$, we plot the wall time of DAT-Graph for various values of $N$.
    B) For $N=100$ we plot the wall time of DAT-Graph for various values of $s$. Error bars are standard deviations over 3 replicates.
    The legend is the same as that of Fig.~\ref{fig:observation}.} \label{fig:timing}
\end{figure}

\begin{figure}[H]
    \centering
    \subfigure[][\label{fig: obs s6 skel} Skeleton]{\includegraphics[width=0.2\columnwidth]{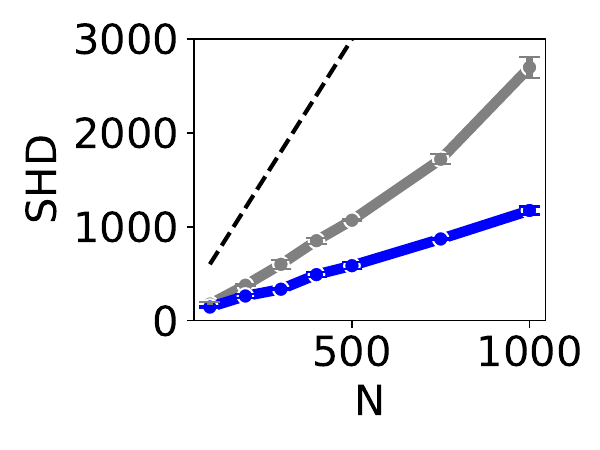}}
    \subfigure[][\label{fig: obs s6 graph} Directed graph]{\includegraphics[width=0.2\columnwidth]{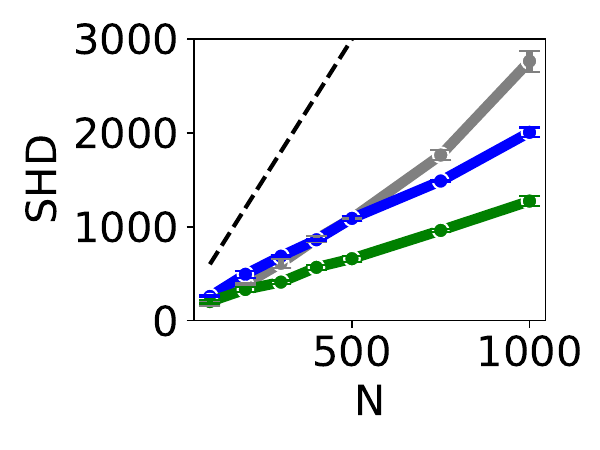}}
    \caption{\textbf{DAT-Graph and the hybrid method learn large graphs accurately even when they are dense.} We perform the experiment in Fig.~\ref{fig:observation} with $s=6$.
    We plot the mean SHD and standard deviation across 3 replicates.
    The legend is the same as that of Fig.~\ref{fig:observation}.} \label{fig:observation s6}
\end{figure}

\begin{figure}[H]
    \centering
    \subfigure[][\label{fig: obs lr1 skel} Skeleton]{\includegraphics[width=0.2\columnwidth]{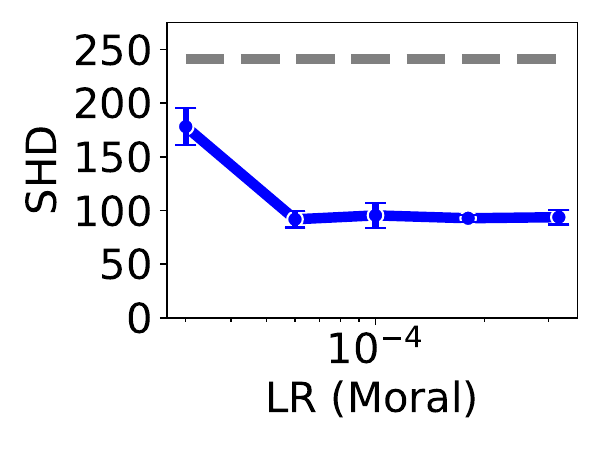}}
    \subfigure[][\label{fig: obs lr1 graph} Directed graph]{\includegraphics[width=0.2\columnwidth]{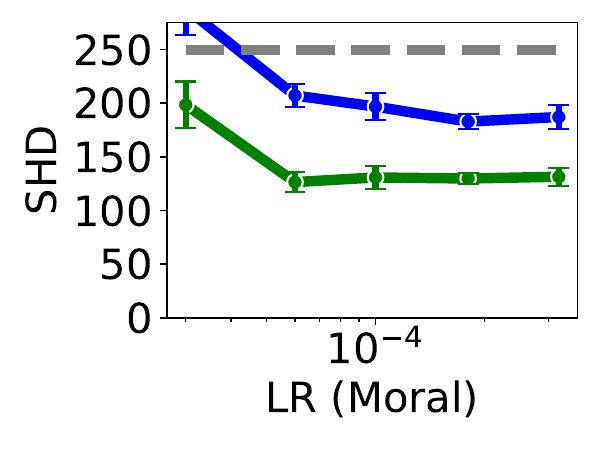}}
    \\
    \subfigure[][\label{fig: obs lr2 skel} Skeleton]{\includegraphics[width=0.2\columnwidth]{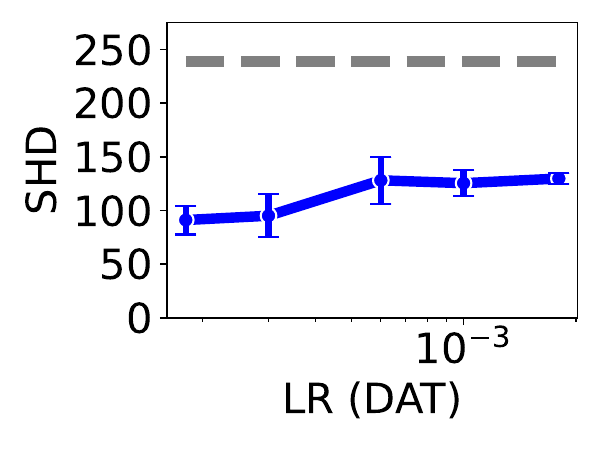}}
    \subfigure[][\label{fig: obs lr2 graph} Directed graph]{\includegraphics[width=0.2\columnwidth]{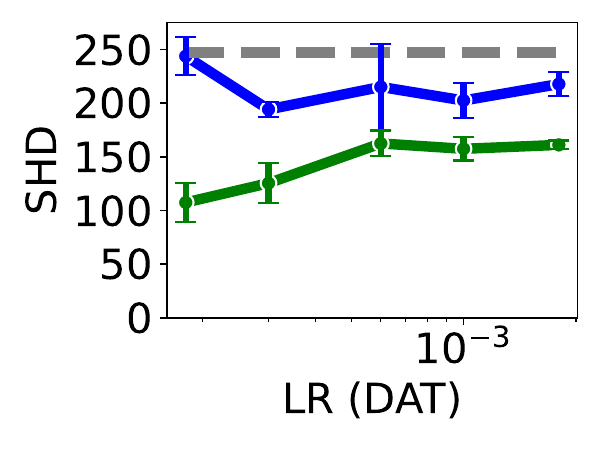}}
    \\
    \subfigure[][\label{fig: obs hw skel} Skeleton]{\includegraphics[width=0.2\columnwidth]{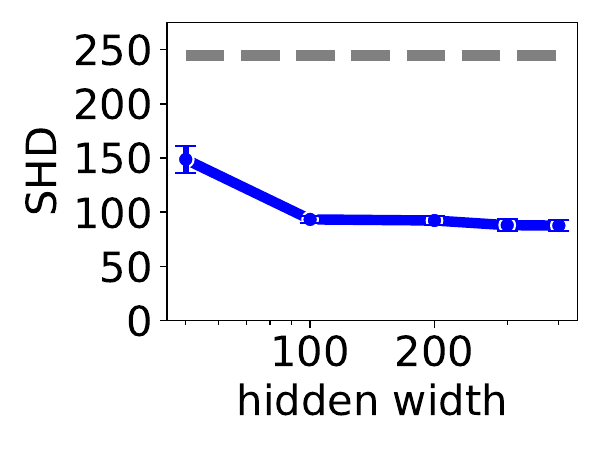}}
    \subfigure[][\label{fig: obs hw graph} Directed graph]{\includegraphics[width=0.2\columnwidth]{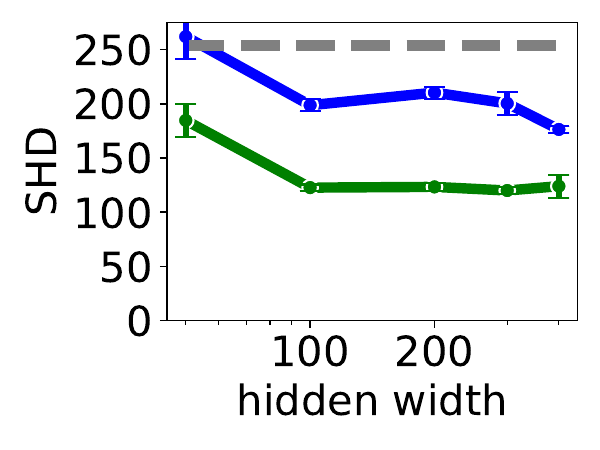}}
    \\
    \subfigure[][\label{fig: obs lm skel} Skeleton]{\includegraphics[width=0.2\columnwidth]{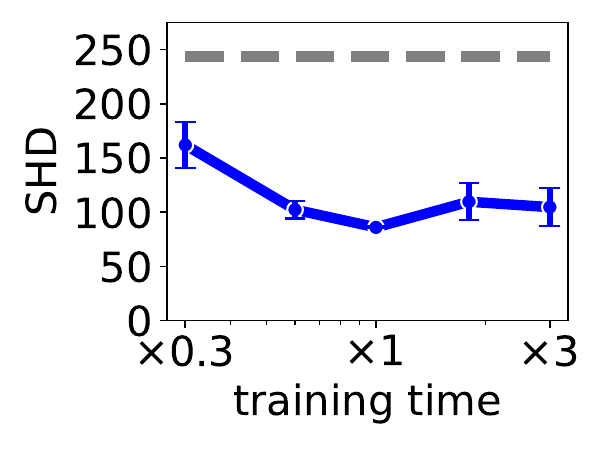}}
    \subfigure[][\label{fig: obs lm graph} Directed graph]{\includegraphics[width=0.2\columnwidth]{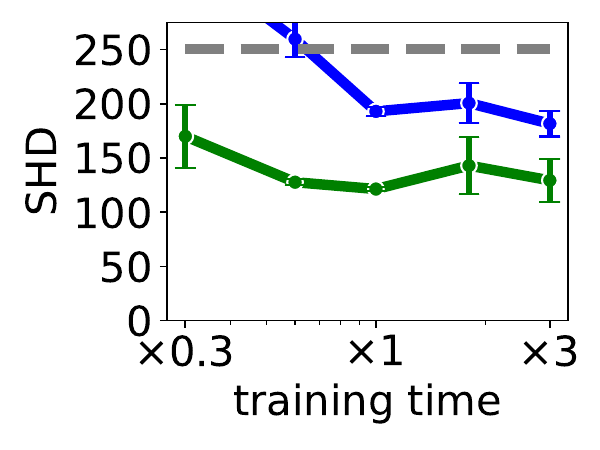}}
    \caption{\textbf{DAT-Graph is robust to the choice of neural network hyperparameters.} We perform the experiment in Fig.~\ref{fig:intervention} with $N=200$, intervening on $50\%$ of variables, varying the learning rate of the moral graph learning step, the DAT step, the hidden widths of the nerual networks, and the neural network training times over an order of magnitude.
    We plot the mean SHD and standard deviation across 3 replicates.
    The legend is the same as that of Fig.~\ref{fig:observation}.} \label{fig:robust nn}
\end{figure}

\begin{figure}[H]
    \centering
    \subfigure[][\label{fig: obs eta1 skel} Skeleton]{\includegraphics[width=0.2\columnwidth]{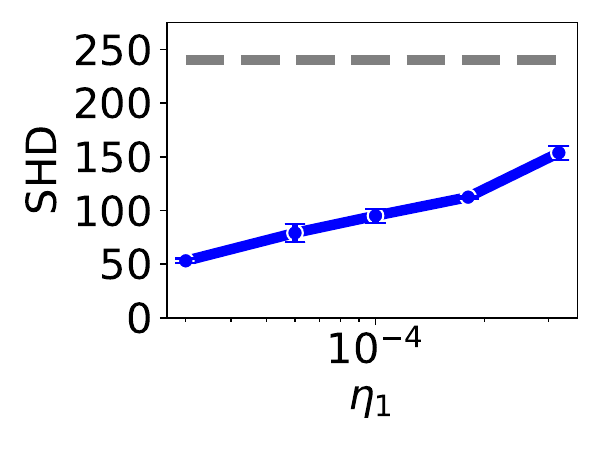}}
    \subfigure[][\label{fig: obs eta1 graph} Directed graph]{\includegraphics[width=0.2\columnwidth]{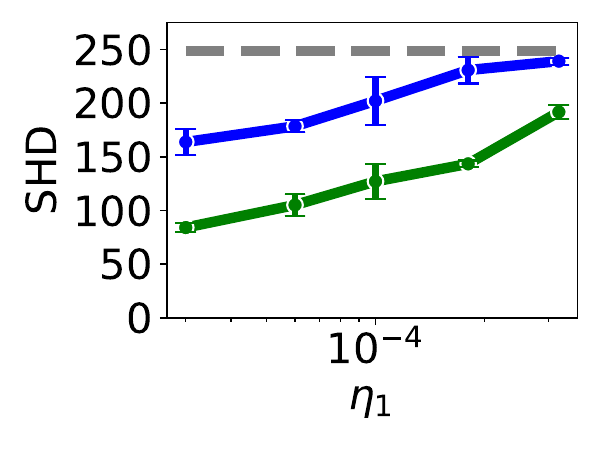}}
    \\
    \subfigure[][\label{fig: obs eta2 skel} Skeleton]{\includegraphics[width=0.2\columnwidth]{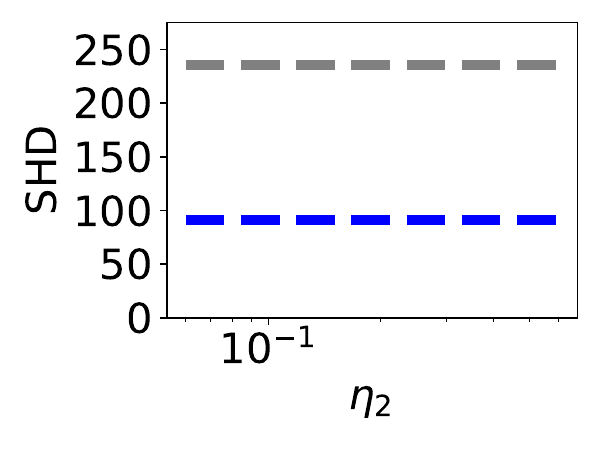}}
    \subfigure[][\label{fig: obs eta2 graph} Directed graph]{\includegraphics[width=0.2\columnwidth]{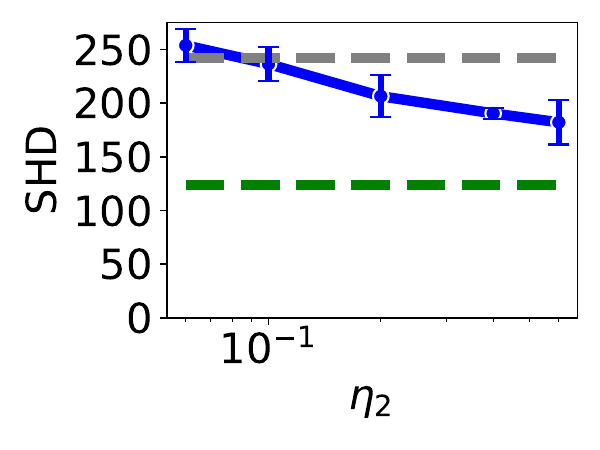}}
    \\
    \subfigure[][\label{fig: obs eta3 skel} Skeleton]{\includegraphics[width=0.2\columnwidth]{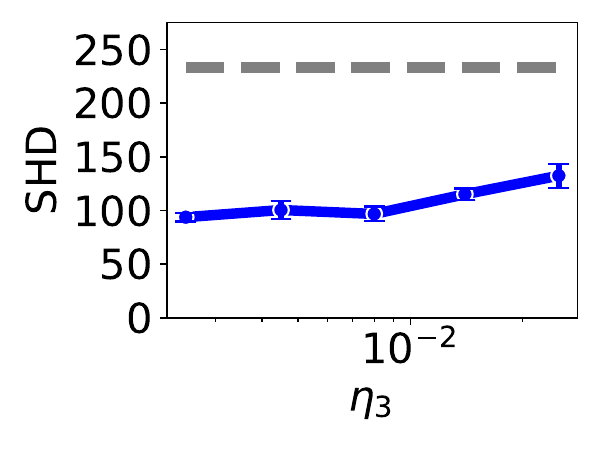}}
    \subfigure[][\label{fig: obs eta3 graph} Directed graph]{\includegraphics[width=0.2\columnwidth]{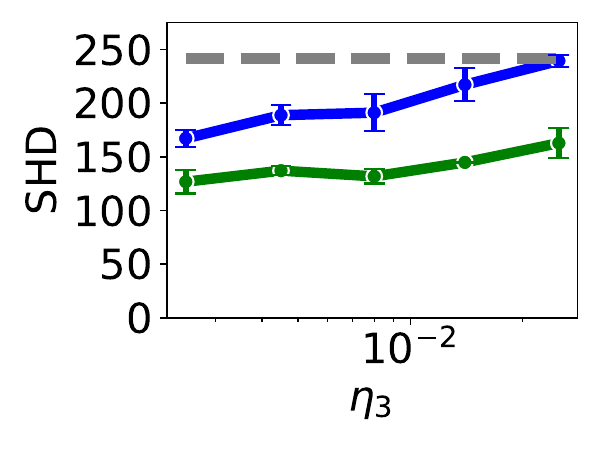}}
    \\
    \subfigure[][\label{fig: obs eta4 skel} Skeleton]{\includegraphics[width=0.2\columnwidth]{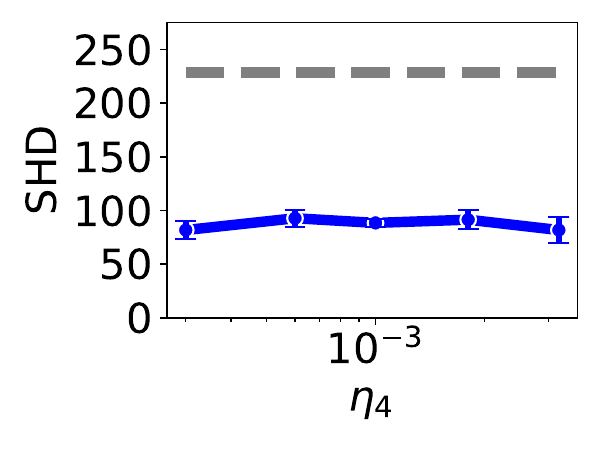}}
    \subfigure[][\label{fig: obs eta4 graph} Directed graph]{\includegraphics[width=0.2\columnwidth]{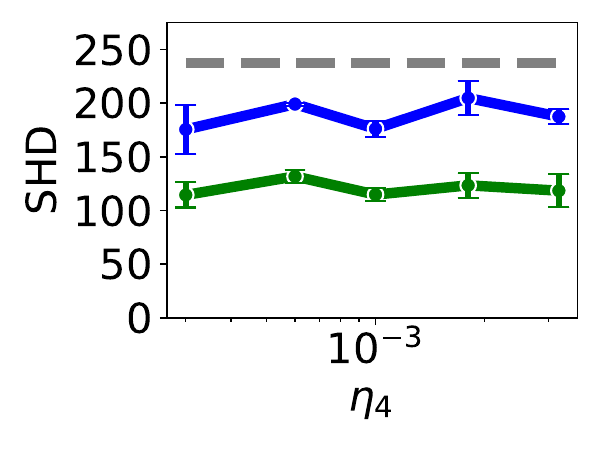}}
    \caption{\textbf{DAT-Graph is robust to the choice of threshold hyperparameters.} We perform the experiment in Fig.~\ref{fig:intervention} with $N=200$, intervening on $50\%$ of variables, varying the parameters $\eta_1, \eta_2, \eta_3, \eta_4$ over an order of magnitude.
    We plot the mean SHD and standard deviation across 3 replicates.
    The legend is the same as that of Fig.~\ref{fig:observation}.} \label{fig:robust eta}
\end{figure}

\begin{figure}[H]
    \centering
    \subfigure[][\label{fig: c skel} Skeleton]{\includegraphics[width=0.2\columnwidth]{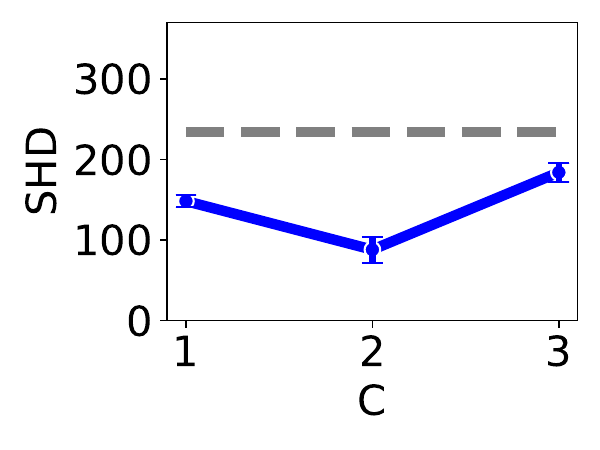}}
    \subfigure[][\label{fig: c graph} Directed graph]{\includegraphics[width=0.2\columnwidth]{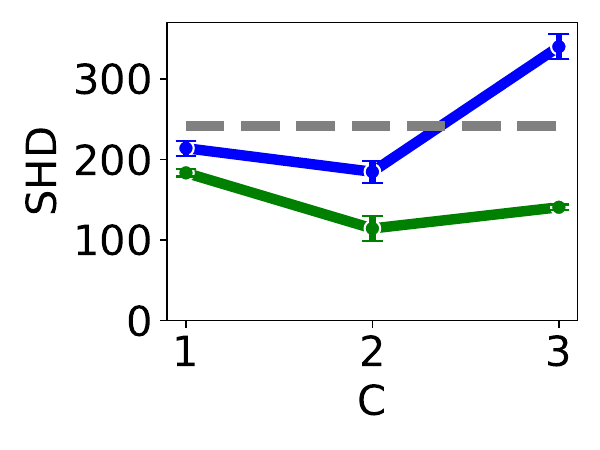}}
    \caption{\textbf{DAT-Graph best learns graphs when using two moments in its variance explained statistic.} We perform the experiment in Fig.~\ref{fig:intervention} with $N=200$ intervening on 50\% of variables.
    We use DAT with the statistic $\sum_{c=1}^CE\mathbb V(T_c(X);Y|\tilde Z_\psi)$ where $T_c(x)=x^c$.
    We plot the mean SHD and standard deviation across 3 replicates.
    The legend is the same as that of Fig.~\ref{fig:observation}.} \label{fig:robust c}
\end{figure}

\subsection{Experiments with small $N$}\label{sec: small N}
Here we show that DAT-Graph can accurately learn graphs at small $N$.
We compare to a classical testing approach with a flexible conditional independence test (PC), a classical explicit graph search procedure (GES), a classical hybrid method (CAM), and a modern method that make learns conditional independence relations by looking for non-linear interactions (NoGAM).

We trained all of these models on data with $N=30$ and $s=3$.
We could not scale NoGAM to learn from 10000 datapoints in under 15 hours of wall time so we trained all models on 6000 datapoints.
The exception is the PC algorithm, which due to the cost of the flexible conditional independence test, could not learn from 600 datapoints in under 15 hours of wall time so we trained this model on 300 datapoints.

In Table~\ref{tab: small N performance} we see that GES, CAM, and NoGAM perform worse than trivial.
This could be because the complex interactions in our data violate the Gaussianity assumption of GES, and the additivity assumption of CAM.
As well, since neural networks with ReLU activations are locally linear, the nonlinearity assumptions in NoGAM are also violated.
The PC algorithm performs slightly better than trivial -- its performance could be limited by its inability to scale to a larger training set size.
On the other hand, DAT-Graph and the Hybrid method are able to learn the graph most accurately.

\begin{table}[H]
\caption{\textbf{DAT-Graph can accurately learn a graph at low $N$.} 
	Mean SHDs and standard deviation for inferred graphs across 3 replicates.
 \label{tab: small N performance}} 
  \setlength\tabcolsep{3pt}
  \begin{center}
\begin{small}
\begin{sc}
\begin{tabular}{ c c c } 
	\toprule
	{Model} & {Skeleton SHD} & {Directed graph SHD}\\ \midrule
        Number edges & 60 & 60\\
	PC (300 datapoints) & 44±2 & 56±5\\
        GES & 99±85 & 104±23\\
        CAM & 79±33 & 87±18\\
        NoGAM & 79±24 & 85±14\\
	SDCD & 25± 4 & 26±4\\
	DAT-Graph & \textbf{8±1} & 16±4\\
        Hybrid & - & \textbf{11± 3}\\ \bottomrule
\end{tabular}
\end{sc}
\end{small}
\end{center}
\end{table}

\subsection{Alternative hybrid model for RNA sequencing data}\label{sec: rna performance second hybrid}

To demonstrate that the benefit of hybridizing DAT-Graph and SDCD did not simply come from restricting the edges learned by SDCD during training,
we also compare to another hybrid model --
in \textbf{SDCD(With Graph)} we use the graph inferred by another SDCD model in place of the one learned by DAT-Graph.
In table~\ref{tab: rna performance second hybrid} we show that the performance of this hybrid model is almost identical to that of SDCD.
\begin{table}[H]
\caption{\textbf{Improved prediction from using DAT-Graph does not come exclusively from training SDCD.} 
	We log the learned mutual information for three datasets as in Fig.~\ref{fig: rna performance}.
 \label{tab: rna performance second hybrid}} 
  \setlength\tabcolsep{3pt}
  \begin{center}
\begin{small}
\begin{sc}
\begin{tabular}{ c c c c } 
	\toprule
	{Model} & {Control} & IFN & {Co-culture}\\ \midrule
	SDCD & 5.2 & 12.7 & 4.3\\
    SDCD(WG) & 5.4  & 12.6 & 4.7\\
	Hybrid & \textbf{26.7} & \textbf{28.3} & \textbf{20.6}\\ \bottomrule
\end{tabular}
\end{sc}
\end{small}
\end{center}
\end{table}

\section{Theory}

\subsection{Proof of Prop.~\ref{prop: np hard}} \label{app: proof np hard}
\begin{proposition}
    \textbf{(Proof of Prop.~\ref{prop: np hard})}
    Even when restricted to the case where $X, Y, Z_1, \dots,Z_M$ are jointly Gaussian with known non-singular covariance matrix, the separating set selection problem is NP-Hard.
\end{proposition}
\begin{proof}
    We will show that the subset sum problem, which is known to be NP-hard, reduces to the above problem.
    The subset sum problem is as follows: given a set of numbers $a_1, \dots, a_M, T\in\mathbb{R}$, is there a subset $S\subset \{a_m\}_{m=1}^M$ such that $\sum_{a\in S}a=T$?

    Let $\epsilon_X, \epsilon_Y, Z_1, \dots, Z_M$ be jointly independent Gaussian variables with variance one, let $X=\epsilon_X+\sum_{m=1}^M Z_m$ and $Y=\epsilon_Y-T\epsilon_X+\sum_{m=1}^M a_mZ_m$.
    Now if $S\subset \{Z_m\}_{m=1}^M$ then $\cov(X, Y|S)=-T+\sum_{Z_m\not \in S}a_m.$
    So if there is a $S\subset \{Z_m\}_{m=1}^M$ such that $X\indep Y|S$ then $0=\cov(X, Y|S)=-T+\sum_{Z_m\not \in S}a_m$ and
    $\sum_{Z_m\not\in S}a_m=T$.
    Similarly if there is no such subset then the answer to the subset sum problem is negative.
\end{proof}

\subsection{Proof of Prop.~\ref{prop: reduce ad to moral}}\label{app: proof reduce ad to moral}
\begin{proposition}
    \textbf{(Proof of Prop.~\ref{prop: reduce ad to moral})}
    Assume $p$ is faithful. $X^n$ and $X^m$ are adjacent in $G$ if and only $X^n\notindep X^m\ |\ U$ for any $U\subset\MB(X^n)\setminus\{X^m\}$.
    If $X^n\indep X^m\ |\ U$ for some $U\subset\MB(X^n)\setminus\{X^m\}$ then $X^n\indep X^m\ |\ U\cup \Pa_G(X^n)\setminus\{X^m\ |\ \Dec_G(X^m)=\emptyset\}$.
\end{proposition}
\begin{proof}
    (Adapted from \citet{Margaritis1999-vu})
    This result is obvious if $X^n$ and $X^m$ are adjacent or $X^m\not\in\MB(X^n)$, so assume $X^m$ is a spouse of $X^n$.
    If $X^m$ is not a descendant of $X^n$, then if $U=\Pa_G(X^n)$ then $X^n\indep X^m\ |\ U$.
    If $X^m$ is a descendant of $X^n$ then include in $U$ $\Pa_G(X^n)$ and all variables in $\Ch_G(X^n)$ that are ancestors of $X^m$ and the parents of these variables.
    Say we have a d-connecting path from $X^n$ to $X^m$.
    Since we have conditioned on all parents of $X^n$ the path must have an arrow out of $X^n$.
    Say the first edge is $X^n\rightarrow X^k$.
    By the definition of $U$, $X^k$ cannot be an ancestor of $X^m$, so the path must eventually encounter its first collider at an $X^l\in U$.
    By the definition of $U$, $X^l$ must be a child of $X^n$ that is an ancestor of $X^m$ or a parent of such a child.
    In either case, $X^l$ is an ancestor of $X^m$.
    Thus, $X^k$ is an ancestor of $X^m$, a contradiction.    
\end{proof}

\subsection{Counter-examples}~\label{app: counterexamples}

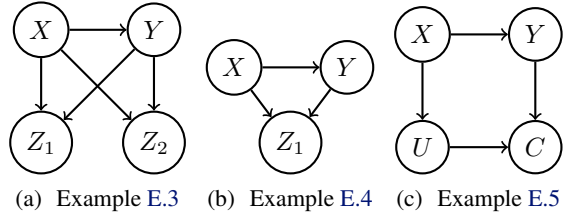
\begin{figure}[H]
    \centering
    \subfigure[][\label{fig: graph 1} Example~\ref{ex: r problem app}]{
        \begin{tikzpicture}[node distance={15mm}, thick, main/.style = {draw, circle}] 
            \centering
            \node[main] (x) {$X$}; 
            \node[main] (y) [right of=x] {$Y$}; 
            \node[main] (z1) [below of=x] {$Z_1$}; 
            \node[main] (z2) [below of=y] {$Z_2$}; 
        
            \draw[->] (x)--(y);
            \draw[->] (x)--(z1);
            \draw[->] (x)--(z2);
            \draw[->] (y)--(z1);
            \draw[->] (y)--(z2);
        \end{tikzpicture}
    }
    \subfigure[][\label{fig: graph 2} Example~\ref{ex: noise problem}]{
        \begin{tikzpicture}[thick, node distance={15mm}, main/.style = {draw, circle}] 
            \centering
            \node[main] (x) {$X$}; 
            \node[main] (y) [right of=x] {$Y$}; 
            \node[main] (z1) at (0.75, -1.) {$Z_1$}; 
        
            \draw[->] (x)--(y);
            \draw[->] (x)--(z1);
            \draw[->] (y)--(z1);
        \end{tikzpicture}
    }
    \subfigure[][\label{fig: graph 3} Example~\ref{ex: dg pro mb app}]{
        \begin{tikzpicture}[node distance={15mm}, thick, main/.style = {draw, circle}] 
            \centering
            \node[main] (x) {$X$}; 
            \node[main] (y) [right of=x] {$Y$}; 
            \node[main] (u) [below of=x] {$U$}; 
            \node[main] (c) [below of=y] {$C$}; 
        
            \draw[->] (x)--(y);
            \draw[->] (x)--(u);
            \draw[->] (u)--(c);
            \draw[->] (y)--(c);
        \end{tikzpicture}
    }
    \caption{\textbf{
    Graphs considered in counterexamples.}} \label{fig: graphs}
\end{figure}

\begin{example}\label{ex: r problem app}
    \textbf{(Proof of Example~\ref{ex: r problem})}
    There are jointly Gaussian variables $X, Y, Z_1, Z_2$ that are faithful to some graph such that $X\notindep Y|\{Z_m\}_{m\in S}$ for any $S\subset\{1, 2\}$ but if $\{r_\psi\}_{\psi\in\Psi}$ is the space of linear functions, there is a $\psi^*$ such that $X\indep Y|r_{\psi^*}(Z_m)_{m=1}^M$.
\end{example}
\begin{proof}
    Let $\epsilon_X, \epsilon_Y, \epsilon_{Z_1}, \epsilon_{Z_2}\sim N(0, 1)$ iid.
    Define $X=\epsilon_X$, $Y=X+\epsilon_Y$, $Z_1=X+Y+\sqrt{\frac{1}{2}}\epsilon_{Z_1}$ and $Z_2=X+Y+\sqrt{3+\frac{1}{2}}\epsilon_{Z_2}$.
    The covariance matrix of $\{X, Y, Z_1, Z_2\}$ is
    \begin{equation*}
    \begin{bmatrix}
    1 & 1 & 2 & 2 \\
    1 & 2 & 3 & 3 \\
    2 & 3 & 5+\frac 1 2 & 5 \\
    2 & 3 & 5 & 8+\frac 1 2
    \end{bmatrix}.
    \end{equation*}
    Then one can calculate that
    \begin{gather}
        \cov(X, Y) = 1, \cov(X, Y|Z_1) = 1 - \frac{6}{5+\frac 1 2}, \cov(X, Y|Z_2) =  1 - \frac{6}{8+\frac 1 2}, \cov(X, Y|Z_1, Z_2) =1-\frac{24}{21+\frac{3}{8}}.
    \end{gather}
    so $X\notindep Y|S$ for any $S\subset \{Z_1, Z_2\}$. Further calculation shows that this is faithful to the graph in Fig.~\ref{fig: graph 1}.
    Now define $Z_3 = \frac 1 2 (Z_1+Z_2)$.
    The covariance matrix of $\{X, Y, Z_3\}$ is 
    \begin{equation*}
    \begin{bmatrix}
    1 & 1 & 2  \\
    1 & 2 & 3  \\
    2 & 3 & 6
    \end{bmatrix}.
    \end{equation*}
    Thus $\cov(X, Y|Z_3) = 0$ so $X\indep Y|Z_3$.
\end{proof}

\begin{example}\label{ex: noise problem}
    There are jointly Gaussian variables $X, Y, Z_1$ that are faithful to some graph such that $X\notindep Y$ and $X\notindep Y\ |\ Z_1$ but if $f_1$ is a Gaussian density then there is a $\psi^*\in[0, 1]$ such that $X\indep Y|\tilde Z_{1, \psi^*}$.
\end{example}
\begin{proof}
    Let $\epsilon_X, \epsilon_Y, \epsilon_{Z_1}\sim N(0, 1)$ iid.
    Define $X=\epsilon_X$, $Y=X+\epsilon_Y$, $Z_1=X+Y+\sqrt{\frac{1}{2}}\epsilon_{Z_1}$.
    The covariance matrix of $\{X, Y, Z_1\}$ is
    \begin{equation*}
    \begin{bmatrix}
    1 & 1 & 2 \\
    1 & 2 & 3 \\
    2 & 3 & 5+\frac 1 2
    \end{bmatrix}.
    \end{equation*}
    Then one can calculate that this distribution is faithful to the graph in Fig.~\ref{fig: graph 2}.
    Now pick $\psi_1^* = 2-\sqrt{2}$ and define $\tilde Z_1 = \psi_1^*Z_1+(1-\psi_1^*)N_1$ where $N_1$ is an independent standard normal.
    $\psi_1^* \epsilon_{Z_1}+(1-\psi_1^*)N_1$ is a mean zero normal distribution of variance
    $$(2-\sqrt{2})^2\frac 1 2 + (\sqrt{2}-1)^2=(2-\sqrt{2})^2\frac 1 2 + (2-\sqrt{2})^2\frac 1 2=(2-\sqrt{2})^2=\psi_1^{*2}.$$
    Thus $\tilde Z_1 =\psi_1^{*}(X+Y+U)$ for some independent standard normal $U$.
    The covariance matrix of $\{X, Y, Z_3\}$ is 
    \begin{equation*}
    \begin{bmatrix}
    1 & 1 & 2\psi_1^{*}  \\
    1 & 2 & 3\psi_1^{*}  \\
    2\psi_1^{*} & 3\psi_1^{*} & 6\psi_1^{*2}.
    \end{bmatrix}.
    \end{equation*}
    Thus $\cov(X, Y|\tilde Z_1) = 0$ so $X\indep Y|\tilde Z_1$.
\end{proof}

\begin{figure}[H]
    \centering
    \subfigure[][\label{fig: dat moral} Moral Graph]{
        \begin{tikzpicture}[node distance={15mm}, thick, main/.style = {draw, circle}] 
            \centering
            \node[main] (x) {$X$}; 
            \node[main] (y) [right of=x] {$Y$}; 
            \node[main] (u) [below of=x] {$U$}; 
            \node[main] (c) [below of=y] {$C$}; 
        
            \draw[-] (x)--(y);
            \draw[-] (x)--(u);
            \draw[-] (u)--(c);
            \draw[-] (y)--(c);
            \draw[-] (y)--(u);
        \end{tikzpicture}
    }
    \subfigure[][\label{fig: dat skel} Skeleton]{
        \begin{tikzpicture}[node distance={15mm}, thick, main/.style = {draw, circle}] 
            \centering
            \node[main] (x) {$X$}; 
            \node[main] (y) [right of=x] {$Y$}; 
            \node[main] (u) [below of=x] {$U$}; 
            \node[main] (c) [below of=y] {$C$}; 
        
            \draw[-] (x)--(y);
            \draw[-] (x)--(u);
            \draw[-] (u)--(c);
            \draw[-] (y)--(c);
        \end{tikzpicture}
    }
    \subfigure[][\label{fig: dat graph} Directed graph]{
        \begin{tikzpicture}[node distance={15mm}, thick, main/.style = {draw, circle}] 
            \centering
            \node[main] (x) {$X$}; 
            \node[main] (y) [right of=x] {$Y$}; 
            \node[main] (u) [below of=x] {$U$}; 
            \node[main] (c) [below of=y] {$C$}; 
        
            \draw[-] (x)--(y);
            \draw[-] (x)--(u);
            \draw[->] (u)--(c);
            \draw[->] (y)--(c);
        \end{tikzpicture}
    }
    \caption{\textbf{
    Stages of DAT-Graph in Example.~\ref{ex: dg pro mb}}} \label{fig: graphs 2}
\end{figure}
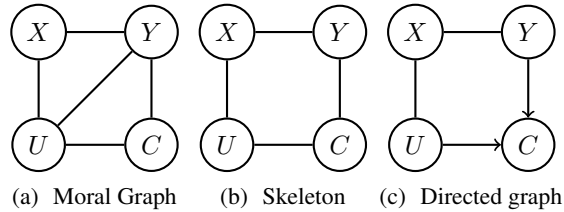

\begin{example}\label{ex: dg pro mb app}
    \textbf{(Proof of Example~\ref{ex: dg pro mb})}
    There are a set of four jointly Gaussian random variables that are not faithful such that DAT-Graph recovers the correct graph.
\end{example}
\begin{proof}
    Let $\epsilon_X, \epsilon_Y, \epsilon_{U}, \epsilon_{C}\sim N(0, 1)$ iid.
    Define $X = \epsilon_X$, $Y = X+\epsilon_Y$, $U = X+\epsilon_U$, $C=2Y+U+\epsilon_C$.
    $\{X, Y, U, C\}$ have covariance matrix
    \begin{equation*}
    \begin{bmatrix}
    1 & 1 & 1 & 3 \\
    1 & 2 & 1 & 5 \\
    1 & 1 & 2 & 4 \\
    3 & 5 & 4 & 15
    \end{bmatrix}.
    \end{equation*}
    $\{X, Y, U, C\}$ are generated according to the graph in Fig.~\ref{fig: graph 3}.
    However, since $\cov(X, Y|C)=0$ this distribution is not faithful to the graph.
    Note it is also not adjacency faithful \cite{Uhler2012-pz}.
    When a distribution is not faithful it can correspond to multiple graph equivalence classes, so the notion of a ``correct'' graph is more delicate.
    We show that the graph in Fig.~\ref{fig: graph 3} is correct in the sense that it is a sparsest Markov graph \cite{Raskutti2018-fo}.
    Then we show that DAT-Graph recovers this ``correct'' graph.
    
    Say $G$ is a graph that is Markov for this distribution, that is, it satisfies Eqn.~\ref{eq: markov eqn} and therefore, for any disjoint $A, B, C\subset \{X, Y, U, C\}$ if $A$ is d-separated from $B$ in $G$ by $C$, then $A\indep B\ |\ C$.
    We have
    $$\cov(X, U)=1, \cov(X, U|Y)=\frac 1 2, \cov(X, U|C)=\frac 3 {15}, \cov(X, U|Y, C)=\frac{59}{55}$$
    so $X$ must be connected to $U$ in $G$.
    Next we have
    $$\cov(Y, C)=5, \cov(Y, C|X)=2, \cov(Y, C|U)=3, \cov(Y, C|X, U)=2$$
    so $Y$ must be connected to $C$ in $G$.
    Finally we have
    $$\cov(U, C)=4, \cov(U, C|X)=1, \cov(U, C|Y)=\frac 3 2, \cov(U, C|X, Y)=1$$
    so $U$ must be connected to $C$ in $G$.
    These three cannot be the only edges in $G$ however as $\cov(X, Y), \cov(X, Y|U)\neq 0$.
    Thus the graph in Fig.~\ref{fig: graph 3} is a sparsest Markov graph.

    We now show that DAT-Graph recovers the graph in Fig.~\ref{fig: graph 3}.
    Some calculations show
    $$\MB(X)=\{U, Y\}, \MB(Y)=\{X, C, U\}, \MB(U)=\{X, C, Y\}, \MB(C)=\{U, Y\}.$$
    DAT-Graph joins nodes $Z_1, Z_2$ in the moral graph if $Z_1\in\MB(Z_2)$ or $Z_2\in\MB(Z_1)$, so we get the moral graph in Fig.~\ref{fig: dat moral}.
    
    Next we learn the skeleton.
    Since $\cov(U, Y|X)=0$, the edge between $Y$ and $U$ is removed with $\sepset(Y, U)=\{X\}$;
    Note $\cov(U, Y|X, C), \cov(U, Y|C)\neq 0$ so $C$ can not be in $\sepset(Y, U)$.
    There is a subset of $\MB(Y)$ that makes $X$ and $Y$ conditionally independent, namely $\cov(X, Y|C)=0$.
    However, $\cov(X, Y),\cov(X, Y|U)\neq 0$ so there is no such subset of $\MB(X)$.
    DAT-Graph only removes an edge if both adjacency tests decide there is no edge.
    Therefore we learn the skeleton in Fig.~\ref{fig: dat skel}.

    Finally, DAT-Graph uses the removed edge between $Y$ and $U$ to find v-structures.
    Since $C\not\in\sepset(Y, U)$, we must have a v-structure $Y\rightarrow C\leftarrow U$.
    All other triplets are labelled non-v-structures.
    Thus we learn the graph equivalence class in Fig.~\ref{fig: dat graph}.
    This is the equivalence class of the ``correct'' graph Fig.~\ref{fig: dat skel}.
\end{proof}

\subsection{Proof of Theorem~\ref{Thm: main reliability}}\label{app: main proof}

In this section we will consider a set of random variables $X, Y, Z_1, \dots, Z_M$ in $\Reals$.
Calling $Z = (Z_m)_{m=1}^M$, we are interested in the separating set selection problem -- evaluating if there is a ``separating set'' of variables $S\subset Z$ such that $X\indep Y|S$.
This is a challenging problem because it is discrete so, we relax it into the separating representation search.
We define independent noise variables $N_1, \dots, N_M$ and define noised versions of $Z$ by $\tilde Z_{m, \psi_m} = \psi_m Z_m + (1-\psi_m)N_m$ where $\psi_m\in[0, 1]$ are variables that control how much information about $Z_m$ we obtain by observing $\tilde Z_{m, \psi_m}$.
We now are interested in finding $\psi=(\psi_m)_m$ such that $X\indep Y|\tilde Z_\psi.$

If $S$ is a separating set then picking $\psi_m=\mathbbm{1}(Z_m\in S)$ gives $X\indep Y|\tilde Z_\psi$ so there is also a separating representation.
However, the contrary might not be true as shown Example~\ref{ex: noise problem}
Here in Thm.~\ref{thm: main thm} we show that if $(f_m)_m$ have thick enough tails then if there is a separating representation there is also a separating set.
Furthermore, if $X\indep Y|\tilde Z_\psi$ then we can recover a separating set $\{Z_m\}_{\psi_m=1}$.

The idea of the proof is that if $f_m$ have thick tails then the values of the noise $N_m$ can be large.
Then if we observe a large $\tilde Z_{m, \psi_m}$ all we can conclude is that $N_m$ took a large value -- we learn little about $Z_m$.
Thus, conditioning on $(\tilde Z_{m, \psi_m})_m$ is similar to observing $Z_m$ if $\psi_m=1$ and not observing $Z_m$ if $\tilde Z_{m, \psi_m}$ is large and $\psi_m<1$.
Thus, if $X\indep Y|\{\tilde Z_{m, \psi_m}\}_m$ then $X\indep Y|\{ Z_{m}\}_{\psi_m=1}$.

\paragraph{Statement and verification of assumption tails of $f_1, \dots, f_M$}
We make the following assumption.
\begin{assumption}\label{ass: tail assump}
    \textbf{($f_m$ has thicker tails than $p(z)$)}
    We assume $(f_m)_m$ are positive bounded symmetric functions on $\mathbb R$ that are decreasing and piece-wise differentiable on $(0, \infty)$.
    \begin{itemize}
        \item We assume the derivative of $\log f_m$ approaches $0$ and for any value $z_m>0$, $(\log f_m)'(z_m) \leq (\log f_m)'(z_m+1)$.
        \item For any subset $S\subset\{1, \dots, M\}$ and almost any set of values $(z_m)_{m\not\in S}$ define the tail probabilities $p(L)=p(\Vert (z_m)_{m\in S}\Vert_\infty\geq L|(z_m)_{m\not\in S})$.
        We assume that for any set of positive numbers, $(h_m)_{m\in S}$,
        \begin{equation}\label{eq: thin tails}
            \sum_{L=0}^{\infty}\frac{p(L+R)}{\prod_{m=1}^M f_m(h_m(L+2R+1))}\to 0\text{ as }R\to\infty.
        \end{equation}
    \end{itemize}

\end{assumption}
The condition on the derivative of $\log f_m$ assumes that $f_m$ is thicker than $f_m(z_m)\propto \exp(-|z_m|)$.
Equation~\ref{eq: thin tails} assumes that the tail of $p(z)$ decreases faster than the tail of $f_m$.
It essentially assumes that ``$f_m$ has thicker tails than $p(z)$''.

Although its statement is technical, the assumption is easy to satisfy because A) we often have some idea of the tails of $p$, or can measure them, and B) we can pick $f$ to have tails as thick as we would like.

We are often willing to assume that $p$ is sub-Gaussian or sub-exponential; for example, RNA counts, based on knowledge of the biological generating process, are regularly assumed to come from a Poisson or negative binomial distribution, which are sub-exponential \citep{Lopez2018-tb}.
Below in Appendix~\ref{sec: assump in exp} we verify the assumption in the setting of our experiments in Section~\ref{sec: experiments} using only the fact that $p$ is sub-Gaussian; a similar argument can be used to verify the assumption in the case that $p$ is sub-exponential. 
When one is not willing to make such an assumption, there are a number of methods to estimate the tails of a distribution. 
Once we have an idea of the thickness of $p$’s tails, we can pick $f$ to be as thick as necessary. 
We hypothesize that there is a tradeoff such that an $f$ with thin tails may lead to separating representations when there are no separating sets, while an $f$ with very thick tails may make training unstable, or reduce statistical efficiency; thus in our experiments in Section~\ref{sec: experiments} we pick an $f$ that is just thick enough to satisfy the assumption.

\paragraph{Statement and proof of theorem}
To construct our noised variables $\tilde Z$, we first pick our noise densities $f_1, \dots, f_M$ and our noise parameters $\psi_1, \dots, \psi_M\in [0, 1]$.
Then we observe noised variables $\tilde z_m = \psi_m z_m+(1-\psi_m)n_m$ where $n_m\sim f_m(n_m)dn_m$.
Thus, if $\psi_m<1$, $p(\tilde z_m|z_m)\propto f_m\left(\frac{1}{1-\psi_m}\tilde z_m-\frac{\psi_m}{1-\psi_m}z_m\right)$.
Thus we get a posterior
$$p((z_m)_m|(\tilde z_m)_m, (\psi_m)_m, (f_m)_m)\propto \prod_{m\ |\ \psi_m\neq 1}f_m\left(\frac{1}{1-\psi_m}\tilde z_m-\frac{\psi_m}{1-\psi_m}z_m\right)dp((z_m)_{\psi_m<1}|(z_m)_{\psi_m=1}).$$

Now we can prove the theorem.
\begin{theorem}\label{thm: main thm}
    Assume Assumption~\ref{ass: tail assump}.
    If there is a $(\psi_m)_m$ such that $X\indep Y|\tilde Z_\psi$ then $X\indep Y|\{Z_m\}_{\psi_m=1}$.
\end{theorem}

\begin{proof}
    Assume $X\indep Y|\tilde Z_\psi$.
    We will show that as $(\tilde z_m)_{\psi_m<1}$ get  large, the posterior converges to the marginal in total variation
    $$p((z_m)_m|(\tilde z_m)_m, (\psi_m)_m, (f_m)_m)\to p((z_m)_{\psi_m<1}|(z_m)_{\psi_m=1}).$$
    Then for any measurable sets $A, B$ we have
    \begin{equation*}
        \begin{aligned}
            0=&p(X\in A, Y\in B|(\tilde z_m)_m, (\psi_m)_m, (f_m)_m)-p(X\in A|(\tilde z_m)_m, (\psi_m)_m, (f_m)_m)p(X\in A|(\tilde z_m)_m, (\psi_m)_m, (f_m)_m)\\
            \to& p(X\in A, Y\in B|(z_m)_{\psi_m=1})-p(X\in A|(z_m)_{\psi_m=1})p(X\in A|(z_m)_{\psi_m=1}).
        \end{aligned}
    \end{equation*}
    This proves that $X\indep Y\ |\ \{Z_m\}_{\psi_m=1}$.

    We fix the values $(z_m)_{\psi_m=1}$ drop the dependence on these conditioned variables.
    We also drop the dependence on $\psi_m$ by redefining $f_m(z_m)=f_m(\frac{\psi_m}{1-\psi_m}z_m)$ and $a_m = \frac{1}{1-\psi_m}\tilde z_m$.
    We set all $a_m$ equal to a single large integer $a$.
    Thus we write the posterior
    $$p(z|a)\propto \prod_{m-1}^M f_m(z_m-a)dp(z).$$
    
    Now we show that this posterior converges in total variation to the marginal $p(z)$ as $a\to\infty$.
    For any $R>0$,
    \begin{equation}\label{eq: two terms}
        \begin{aligned}
            \int \left\vert\frac{\prod_{m=1}^M f_m(z_m-a)}{\int \prod_{m=1}^M f_m(z_m-a)dp(z)}-1\right\vert dp(z)=&\int_{\Vert z\Vert_\infty\leq R} \left\vert\frac{\prod_{m=1}^M f_m(z_m-a_m)}{\int_{\Vert z\Vert_\infty\leq R} \prod_{m=1}^M f_m(z_m-a)dp(z)}-1\right\vert dp(z)\\
            &+\int_{\Vert z\Vert_\infty\leq R}\bigg|\frac{\prod_{m=1}^M f_m(z_m-a)}{\int_{\Vert z\Vert_\infty\leq R} \prod_{m=1}^M f_m(z_m-a)dp(z)}\\
            &\ \ \ \ \ \ \ \ \ \ \ \ \ \ \ \ \ \ \ \ \ \  -\frac{\prod_{m=1}^M f_m(z_m-a)}{\int  \prod_{m=1}^M f_m(z_m-a)dp(z)}\bigg|dp(z)\\
            &+ \int_{\Vert z\Vert_\infty> R} \left\vert\frac{\prod_{m=1}^M f_m(z_m-a)}{\int \prod_{m=1}^M f_m(z_m-a)dp(z)}\right\vert dp(z)\\
            =&\int_{\Vert z\Vert_\infty\leq R} \left\vert\frac{\prod_{m=1}^M f_m(z_m-a)}{\int_{\Vert z\Vert_\infty\leq R}  \prod_{m=1}^M f_m(z_m-a)dp(z)}-1\right\vert dp(z)\\
            &+2\frac{\int_{\Vert z\Vert_\infty> R}  \prod_{m=1}^M f_m(z_m-a)dp(z)}{\int  \prod_{m=1}^M f_m(z_m-a)dp(z)}.\\
        \end{aligned}
    \end{equation}
    We now show both of these terms vanish as $R, a\to\infty$.
    
    For the first term, note for all $\Vert z\Vert_\infty\leq R$.
    $$\prod_{m=1}^M \frac{f_m(a+R)}{f_m(a-R)}\leq \frac{\prod_{m=1}^M f_m(z_m-a)}{\int_{\Vert z\Vert_\infty\leq R}\prod_{m=1}^M f_m(z_m-a)dp(z)}\leq \frac{1}{p(\Vert z\Vert_\infty\leq R)}\prod_{m=1}^M \frac{f_m(a-R)}{f_m(a+R)}$$
    By our assumption on the derivatives of $\log f_m$,
    \begin{equation*}
        \begin{aligned}
        \frac{f_m(a-R)}{f_m(a+R)} =& \exp\left(\log f_m(a-R) - \log f_m(a+R)\right)\\
        =&\exp\left(-\int_{a-R}^{a+R}(\log f_m)'(s)ds\right).
        \end{aligned}
    \end{equation*}
    This quantity $\to 1$ if $R\to\infty$ slowly enough and $a\to\infty$. 
    Thus, if $R\to\infty$ slowly enough and $a\to\infty$, the first term in Eqn.~\ref{eq: two terms} is approaches
    $$\int_{\Vert z\Vert_\infty\leq R} \left\vert\frac{1}{p(\Vert z\Vert_\infty\leq R)}-1\right\vert dp(z)=p(\Vert z\Vert_\infty\geq R)\to 0.$$
    
    For the second term, first note that if $L_2>L_1$ then
    \begin{equation*}
        \begin{aligned}
        \frac{f_m(L_1)}{f_m(L_2)} =& \exp\left(\log f_m(L_1) - \log f_m(L_2)\right)\\
        =&\exp\left(-\int_{L_1}^{L_2}(\log f_m)'(s)ds\right)\\
        \leq &\exp\left(-\int_{0}^{L_2-L_1}(\log f_m)'(s)ds\right)\\
        =&\frac{f_m(0)}{f_m(L_2-L_1)}.
        \end{aligned}
    \end{equation*}
    Assume we have picked $R$ large enough that $p(\Vert z\Vert_\infty\leq R)\geq 1/2$.
    Define the annulus probabilities $\tilde p(L)=p(L+1>\Vert (z_m)_{m\in S}\Vert_\infty\geq L)$
    The second term can be broken up into annuluses
    \begin{equation*}
        \begin{aligned}
            \frac{\int_{\Vert z\Vert_\infty> R}  \prod_{m=1}^M f_m(z_m-a)dp(z)}{\int  \prod_{m=1}^M f_m(z_m-a)dp(z)}\leq &\sum_{L=R}^\infty \frac{\int_{L\leq \Vert z\Vert_\infty< L+1}  \prod_{m=1}^M f_m(z_m-a)dp(z)}{\int_{\Vert z\Vert_\infty\leq R}  \prod_{m=1}^M f_m(z_m-a)dp(z)}\\
            \leq &\sum_{L=R}^{a-1} \frac{p(L\leq \Vert z\Vert_\infty< L+1) \prod_{m=1}^M f_m(a-(L+1))}{p(\Vert z\Vert_\infty\leq R)  \prod_{m=1}^M f_m(a+R)}\\
            &+\sum_{L=a}^\infty \frac{p(L\leq \Vert z\Vert_\infty< L+1) \prod_{m=1}^M f_m(L-a)}{p(\Vert z\Vert_\infty\leq R)  \prod_{m=1}^M f_m(a+R)}\\
            \leq&2p(\Vert z\Vert_\infty\leq R)^{-1}\sum_{L=R}^{a-1} \tilde p(L) \prod_{m=1}^M \frac{f_m(a - (L+1))}{f_m(a+R)}\\
            &+2p(\Vert z\Vert_\infty\leq R)^{-1}\sum_{L=a}^{2a-R-1} \tilde p(L) \prod_{m=1}^M\frac{ f_m(L-a)}{f_m(a+R)}\\
            &+\frac{p(\Vert z\Vert_\infty\geq 2a+R)}{p(\Vert z\Vert_\infty\leq R)}\prod_{m=1}^M\frac{ f_m(a-R)}{f_m(a+R)}\\
            \leq&2p(\Vert z\Vert_\infty\leq R)^{-1}\sum_{L=R}^{a-1} \tilde p(L) \prod_{m=1}^M \frac{f_m(0)}{f_m(L+R+1)}\\
            &+2p(\Vert z\Vert_\infty\leq R)^{-1}\sum_{L=a}^{2a-R-1} \tilde p(L) \prod_{m=1}^M\frac{ f_m(0)}{f_m(2a+R-L)}\\
            &+o_{a\to\infty}(1)\\
            \leq &\frac{2}{p(\Vert z\Vert_\infty\leq R)}\left(\prod_{m=1}^Mf_m(0)\right)\sum_{L=0}^{a-R-1}\frac{\tilde p(L+R)+\tilde p(2a-R-1-L)}{\prod_{m=1}^M f_m(L+2R+1)}+o_{a\to\infty}(1)\\
            \leq &\frac{2}{p(\Vert z\Vert_\infty\leq R)}\left(\prod_{m=1}^Mf_m(0)\right)\sum_{L=0}^{\infty}\frac{p(L+R)}{\prod_{m=1}^M f_m(L+2R+1)}+o_{a\to\infty}(1).\\
        \end{aligned}
    \end{equation*}
    By our assumption, this vanishes as $a\to\infty$ and $R\to\infty$ slowly enough.
\end{proof}

\subsubsection{Assumption~\ref{ass: tail assump} in the Experiments} \label{sec: assump in exp}

In this section we verify that Assumption~\ref{ass: tail assump} is satisfied in our experiments.
The noise we choose is $f_m\sim \frac 1 2g(\mathrm{Laplace})$ where $g(z_m)=z_m$ if $z_m\leq 1$ and $g(z_m)=\mathrm{sgn}(z_m)|z_m|^{1.1}$ if $z_m\geq 1$.
This has continuous density with
$f_m(z_m)\propto (2z_m)^{-\frac{0.1}{1.1}}\exp\left(-|2z_m|^{\frac 1 {1.1}}\right)$ if $|z_m|\geq \frac 1 2$ and $f_m(z_m)\propto \exp\left(-|2z_m|\right)$ if $|z_m|\leq \frac 1 2$. 
and one can check that the derivative of the logarithm is increasing almost everywhere on $(0, \infty)$ and approaches $0$.
The derivative has a jump discontinuity at $\frac 1 2$ but one can easily check that for any $\frac 1 2 > z_m>0$,
$$(\log f)'(z_m) =-2 <  -\frac{0.1}{1.1} - \frac{2^{1/1.1}}{1.1}= (\log f)'(1) \leq (\log f)'(z_m+1).$$

In our simulation, variables $X^n$ are made up of neural networks $\tilde h_n$ applied to Gaussian variables.
Neural networks are Lipschitz, so each $X^n$ is sub-Gaussian.
This is also the case when conditioning since if $E[\exp(tX^{n2})]<\infty$ for some $t>0$ then if we pick some set $S$, $E[E[\exp(tX^{n2})|(X^m)_{m\in S}]]=E[\exp(tX^{n2})]<\infty$ so $E[\exp(tX^{n2})|(X^m)_{m\in S}]<\infty$ for almost every value of $(X^m)_{m\in S}$. 
Therefore $p(L+R)\leq \sum_{n\not\in S}p(|X^n|>L+R|(X^m)_{m\in S})$ decays with negative square exponential tails, much faster than $f_m(L+2R-1)$.
For some $C>0$, any set of positive numbers, $(h_m)_m$, and large enough $R$,
\begin{equation*}
    \sum_{L=0}^{\infty}\frac{p(L+R)}{\prod_{m=1}^M f_m(h_m(L+2R+1))}\lesssim \sum_{L=0}^{\infty}\frac{\exp(-C(L+R)^2)}{\exp(-(\sum_m h_m)(L+2R))}
    =\sum_{L=0}^{\infty}\exp\left(-C(L+R)^2+(L+2R)\sum_m h_m\right).
\end{equation*}
This sum is clearly finite and for large enough $R$, each term is decreasing in $R$.
Therefore, the sum converges to $0$ as $R\to\infty$.
Thus Assumption~\ref{ass: tail assump} is satisfied in our experiments.

\subsubsection{Discussion of Theorem~\ref{Thm: main reliability}}\label{sec: thm disscussion}
Theorem~\ref{Thm: main reliability} states that if we choose our noise densities $(f_m)_m$ to have thick tails then we can answer the separating set selection problem by answering the separating representation search problem and setting $\sepset(X, Y)=\{Z_m\}_{\psi_m=1}$.
It however does not suggest that this is the only strategy for choosing $(f_m)_m$ and it does not say anything about when $\psi_m$ can be between $(0, 1)$.

For the answers of the separating set selection problem and the separating representation problem to disagree we must have a $(\psi_m)_m$ such that for all measurable sets $A, B$ and all values $(\tilde z_m)_m$
\begin{equation}\label{eq: constraint}
        0=p(X\in A, Y\in B|(\tilde z_m)_m, (\psi_m)_m, (f_m)_m)-p(X\in A|(\tilde z_m)_m, (\psi_m)_m, (f_m)_m)p(X\in A|(\tilde z_m)_m, (\psi_m)_m, (f_m)_m).
\end{equation}
This is an infinite set of constraints that must be satisfied by a value of an $M$-dimensional parameter $(\psi_m)_m$.

In Example~\ref{ex: noise problem} we showed that if the variables are jointly Gaussian and $(f_m)_m$ are also specially chosen to be Gaussian then there can be a separating representation but no separating set.
In this case, two variables are conditionally independent if their conditional covariance, which does not depend on $(\tilde z_m)_m$, is $0$.
Thus the infinitely many constraints in Eq.~\ref{eq: constraint} collapse to one:
$$\mathrm{Cov}[XY|\tilde Z_{\psi}]=0.$$
We conjecture that for generic choices of $(f_m)_m$, the infinitely many constraints in Eq.~\ref{eq: constraint} remain distinct;
they are therefore impossible to satisfy for any value of $\psi$ that is not close to the indicator of a separating set.
In this case we conjecture $\{Z_m\}_{\psi_m^*>c}$ is a separating set for any value of $1-\epsilon > c > \epsilon$.

\end{document}